\documentclass{article}

\usepackage{PRIMEarxiv}

\usepackage[utf8]{inputenc} 
\usepackage[T1]{fontenc}    
\usepackage{hyperref}       
\usepackage{url}            
\usepackage{booktabs}       
\usepackage{amsfonts}       
\usepackage{nicefrac}       
\usepackage{microtype}      
\usepackage{lipsum}
\usepackage{fancyhdr}       
\usepackage{graphicx}       
\graphicspath{{media/}}     

\usepackage{amsmath}
\usepackage{bm}
\usepackage{mathtools}
\usepackage{amssymb}
\usepackage{longtable}
\usepackage{comment}
\usepackage{tikz}
\usepackage{xcolor}
\usepackage{amsthm}
\usepackage{wrapfig}
\usepackage{amsmath}
\usepackage{amssymb}
\usepackage{mathrsfs}
\usepackage{amscd}
\usepackage{amsfonts}

\usepackage{lipsum}  

\usepackage{enumitem}   
\usepackage{wrapfig}    
\usepackage{xcolor}
\usepackage{etoolbox}
\usepackage{tcolorbox}
\usepackage{makecell}
\usepackage{pifont}
\usepackage{graphics,color}

\newtheorem{example}{Example}
\newtheorem{theorem}{Theorem}
\newtheorem{definition}{\textbf{Definition}}
\newtheorem{proposition}{\textbf{Proposition}}

\DeclareMathOperator*{\argmax}{arg\,max}

\def\safe{\mbox{\sl Safe}}

\def\opert{\mbox{\sl Pert}}
\def\pert{\overline{\opert}}
\def\uco{\mbox{\sl uco}}
\def\res{\mathbb{C}}

\newcommand{\tuple}[1]{\langle #1 \rangle}
\def\id{\mbox{\sl id}}

\def\defi{\mbox{\raisebox{0ex}[1ex][1ex]{$\stackrel{\mbox{\tt\tiny def}}{\; =\;}$}}}
\newcommand{\sset}[2]{\left\{~#1  \left |
	\begin{array}{l}#2\end{array}
	\right.     \right\}}
\def\comp{\mathrel{\hbox{\footnotesize${}\!{\circ}\!{}$\normalsize}}}

\usepackage{algorithm}
\usepackage{algorithmic}

\usepackage{tikz,ifthen}
\usetikzlibrary{arrows,trees,backgrounds,automata,shapes,decorations,plotmarks,fit,calc,positioning,shadows,chains}

\definecolor{color0}{RGB}{240, 78, 64} 
\definecolor{color4}{RGB}{60, 220, 125} 
\definecolor{color6}{RGB}{120, 100, 200} 
\definecolor{color7}{RGB}{107, 100, 200} 

\definecolor{nnedgecolor}{RGB}{90,90,90}
\tikzstyle{every pin edge}=[<-,shorten <=1pt]
\tikzstyle{every path}=[draw=color7!50]
\tikzstyle{neuron}=[circle,fill=black!25,minimum size=17pt,inner sep=0pt]
\tikzstyle{input neuron}=[neuron, fill=color4]
\tikzstyle{output neuron}=[neuron, fill=color0]
\tikzstyle{hidden neuron}=[neuron, fill=color6]
\tikzstyle{annot} = [text width=4em, text centered]
\tikzstyle{nnedge} = [-{stealth},shorten >=0.1cm, shorten <=0.05cm,line 
width=0.8pt,nnedgecolor]
\tikzstyle{nnedge_t} = [-{dashed},shorten >=0.1cm, shorten <=0.05cm,line 
width=0.8pt,nnedgecolor]
\usetikzlibrary{calc}

\def\wIm{\overline{\Im}}

\title{Advancing Neural Network Verification through \\ Hierarchical Safety Abstract Interpretation}

\author{
  Luca Marzari, Isabella Mastroeni and Alessandro Farinelli
  \\
  Department of Computer Science, University of Verona, Verona, Italy.\\
  Contact authors: \textit{luca.marzari@univr.it}}

\begin{document}
\maketitle

\begin{abstract}
Traditional methods for formal verification (FV) of deep neural networks (DNNs) are constrained by a binary encoding of safety properties, where a model is classified as either safe or unsafe (robust or not robust). This binary encoding fails to capture the nuanced safety levels within a model, often resulting in either overly restrictive or too permissive requirements. In this paper, we introduce a novel problem formulation called \textsc{Abstract DNN-Verification}, which verifies a hierarchical structure of unsafe outputs, providing a more granular analysis of the safety aspect for a given DNN. Crucially, by leveraging abstract interpretation and reasoning about output reachable sets, our approach enables assessing multiple safety levels during the FV process, requiring the same (in the worst case) or even potentially less computational effort than the traditional binary verification approach. Specifically, we demonstrate how this formulation allows rank adversarial inputs according to their \textit{abstract safety} level violation, offering a more detailed evaluation of the model's safety and robustness. Our contributions include a theoretical exploration of the relationship between our novel abstract safety formulation and existing approaches that employ abstract interpretation for robustness verification, complexity analysis of the novel problem introduced, and an empirical evaluation considering both a complex deep reinforcement learning task (based on Habitat 3.0) and standard DNN-Verification benchmarks.
\end{abstract}


\section{Introduction}

Deep neural networks (DNNs) have significantly advanced the field of artificial intelligence, powering applications from image recognition \cite{image} and robotics medical applications \cite{colon} to autonomous navigation \cite{navigation,curriculum}. 
The success of these approximation functions across various domains has also prompted their use in safety-critical areas where expensive hardware and human safety can be involved. However, the vulnerability of these systems to the so-called adversarial inputs \cite{adversarial,tacas}---imperceptible modifications to the original input data that can lead to wrong and potentially catastrophic decisions--- has raised concerns about their reliability and safety.
To address such an issue, formal verification (FV) of neural networks \cite{Reluplex,singh2019abstract,LiuSurvey} has recently emerged as a promising solution. This process involves defining safety or robustness specifications that can be formally verified to ensure the network behaves as intended, even in the presence of these small input perturbations \cite{LiuSurvey}. In detail, a property to be verified is defined as a precondition on the input space and a postcondition on the output. 
In this context, the use of abstract interpretation \cite{CC77} and, in particular, abstract domain provides a key resource to solve efficiently the problem \cite{Ai2,singh2019abstract}. In fact, the precondition of a safety specification is typically encoded using an abstraction of the input domain, exploiting different geometries, such as convex polytopes, to capture all the possible configurations of an unsafe scenario we aim to verify. The postcondition represents the desired (or undesired) outcome we expect from the neural network for that specific precondition. For instance, considering a scenario where a neural network guides an autonomous system, the precondition of a safety property can encode a specific unsafe situation where the autonomous agent is near an obstacle, and the postcondition is the action that should be avoided in order not to collide. Nonetheless, in many real-world scenarios, reducing the answer of formal verification to a binary classification of \textit{``safe"} or \textit{``unsafe"} nature leads to the risks of losing the full expressive potential of these FV tools. In particular, when encoding a safety property, a natural question arises: \textit{``Are some actions only marginally unsafe and thus potentially acceptable, while others pose severe risks?"} Despite significant advancements of DNN verification tools over the years \cite{LiuSurvey,Reluplex,crown,acrown,bcrown,wei2024modelverification}, the binary encoding of properties ``safe" or ``unsafe" (``robust" or ``not robust") often fails to capture the full range of potential system behaviors, resulting in the verification of properties that are either overly restrictive or excessively permissive. We argue that these limitations stem from the current focus on using abstract interpretation (in particular over-approximations) to derive concrete output results rather than employing abstract reasoning to directly analyze the structure of the output set.
\begin{figure*}[t]
    \begin{center}
        \includegraphics[width=\linewidth]{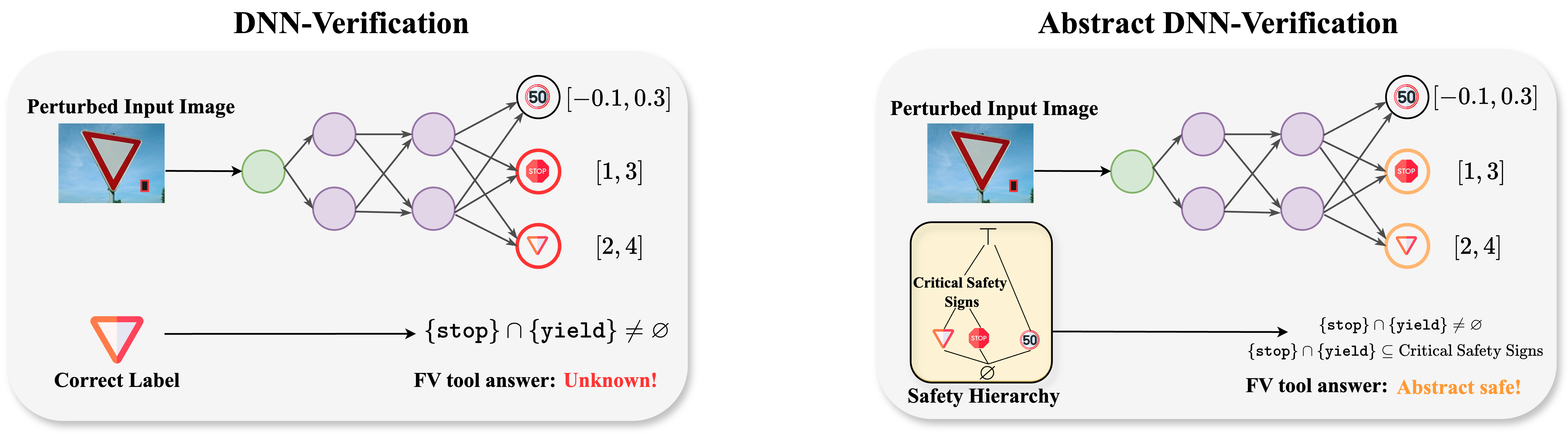}
    \end{center}
    \caption{Overview of the proposed \textsc{Abstract DNN-Verification} in this paper. On the left, given the intersection of the reachable sets, the concrete classes are not able to provide an answer. Conversely if we use the first level of abstraction (right part of the image) even if the reachable sets are overlapped, i.e., the intersection is nonempty, it is fully contained in the set of \textit{``safety-critical signs"} output abstraction, allowing to provide the \textit{abstract safe} answer.}
    \label{fig:overview}
    \vspace{-2mm}
\end{figure*}

To address this limitation, we introduce \textsc{Abstract DNN-Verification}, a novel problem formalization that enables reasoning directly within a hierarchical structure of tolerable unsafe outputs. This approach allows verification of a predefined hierarchy among outputs, providing a deeper understanding of the safety levels of a DNN and potentially helping the verification in ambiguous cases—such as when the over-approximation induced by the propagation of the abstract input domain produces overlapping output results. Specifically, in these situations, incomplete verifiers, which prioritize scalability over completeness, in the worst case, could return an \textit{``unknown"} result, leaving the safety question unresolved. Conversely, sound and complete verifiers always provide an answer but require extensive input refinement \cite{reluval} or splitting of the network's unstable regions \cite{bab,bcrown} to resolve the ambiguities, leading to significant computational overhead, especially for large-scale models. Although our formulation preserves the NP-hard complexity of the original problem \cite{Reluplex}, it introduces a novel capability: in cases with overlapping or ambiguous outputs, it allows analysts to classify the DNN as either \textit{safe} or \textit{abstract safe}. The latter refers to situations where the model's output does not match the exact target class but still falls within a predefined safe region of the output hierarchy specified for the verification process. Furthermore, it facilitates a deeper understanding of the regions or features that greatly impact the safety hierarchy, thus offering more actionable insights into the model’s behavior.

Let us focus on the example in Fig.\;\ref{fig:overview}, and consider a scenario where a mobile robot relies on a neural network to interpret road signs from image inputs to make navigation decisions. For instance, suppose we test the system's robustness to local patch perturbations as depicted in the perturbed input image of Fig.\;\ref{fig:overview}. In this context, the standard FV process passes the perturbed image through the DNN and evaluates the verification outcome on over-approximated reachable output sets by selecting the greater reachable set.\footnote{Typically, the node with greater value is the one selected by the DNN. We will provide further details in the next sections.} However, the over-approximation error often leads to overlapping reachable sets, as shown in Fig.\;\ref{fig:overview}(left), where $\{\texttt{stop}\} \cap \{\texttt{yield}\} = [1,3] \cap[2, 4]\neq \varnothing$, preventing a precise result without the necessity of further computational demanding operations.
In contrast, with \textsc{Abstract DNN-Verification}, for a given input perturbation, we can compute a set of potential DNN answers and then check if, for a given safety hierarchy, the set respects one or multiple levels of the latter, thus providing the possibility to the analyst to better understand the actual DNN's safety level. In the example, the defined hierarchy (yellow box in Fig \ref{fig:overview} (right)) groups both the $\{\texttt{stop}\}$ and $\{\texttt{yield}\}$ signs into the broader category of \textit{``Critical Safety Signs"}. This hierarchy abstraction acknowledges that while a misclassification between these two signs would technically be an error, it represents a form of error that is still acceptable. Specifically, if the DNN misclassifies a ``yield`` as a ``stop`` sign, the hierarchy abstracts the decision on \textit{``Critical Safety Signs"} which is strictly below the $\top$ (representing the fact that the DNN does not classify the sign as a ``normal" sign, e.g., 50 mph), and thus the system's safety is abstractly preserved.
We emphasize that the definition of the safety property, and the resulting hierarchy, is inherently shaped by the analyst's perspective and contextual judgment. This subjectivity, however, is widely acknowledged and accepted in the neural network verification literature, where it is typically grounded in domain knowledge and tailored to the specific operational context. Hence, our hierarchical approach can be designed at different levels of granularity to enrich the information available when violations occur, which is particularly valuable given the practical challenges of achieving complete model robustness in complex, real-world environments. On the example of misclassifying a ``stop” sign as a ``yield” sign, while this is indeed an error, our framework still offers important safety insight: the prediction remains within a class of potentially dangerous signs, rather than being misclassified as a “safe” sign. This information can be used to trigger post-verification fallback strategies, such as slowing down and stopping (even if it is a ``yield" sign). In contrast, binary verification would mark this case as ``not robust,” providing no further guidance to the analyst. Throughout the paper, we will show that even in safety-critical situations, our method delivers richer and more actionable feedback than binary verification, enabling system designers to understand not only that a failure occurred, but also why it occurred and how to respond appropriately.

Specifically, in the following sections, we guide the reader to understand how our new formulation allows us to answer multiple critical questions during the FV process, such as: \textit{``Which level of unsafety does my model present?"}, \textit{``Does the model output an action that is only potentially safe?"}, and \textit{``How much do perturbations of a particular input feature impact on the overall safety?"} \\
Finally, to further confirm the impact of our new formulation, we apply \textsc{Abstract DNN-Verification} on realistic deep reinforcement learning human-robot cooperation tasks, namely Habitat \cite{habitat2,habitat3} as well as in two standard benchmarks of the international verification of neural networks competition (VNN-COMP) \cite{VNN-comp2023}, such as CIFAR10 \cite{cifar10}. Results in the first experiment show that even relatively simple hierarchies enable in our novel problem formulation allow the ranking of different adversarial attacks based on the violation of one or more levels of the hierarchy, a capability absent in standard approaches that fail to differentiate between varying attack impacts. 
In the second one, for a given safety hierarchy, our approach provides valuable insights into the robustness level of deep neural networks, uncovering scenarios where instances traditionally classified as ``not robust" may exhibit a specific form of robustness when viewed through the lens of tolerable (based on the safety hierarchy) misclassifications. This nuanced understanding of safety and robustness offers a more realistic and practical evaluation of the model's performance in safety-critical applications.

\section{Preliminaries}\label{sec:preliminaries}

This section provides the reader with all the basic knowledge and notation to easily follow the paper.

\paragraph{Deep Neural Network.} Let $f: \mathbb{R}^m \to \mathbb{R}^n$ be a neural network that maps inputs in $\mathbb{R}^m$ to outputs in  $\mathbb{R}^n$, with $m,n \in \mathbb{N}$, respectively. For a generic input $\mathbf{x} = \langle x_1, \dots, x_m \rangle \in \mathbb{R}^m$ the corresponding output is $\mathbf{y} =f(\mathbf{x}) = \langle y_1, \dots, y_n \rangle \in \mathbb{R}^n$. Typically, we are primarily interested in knowing which output value a DNN will select, corresponding to identifying the output node with the maximum value. Let $g:\mathbb{R}^n \to \mathbb{O}$ (with $\mathbb{O}\defi[1,n] \times\mathbb{R}$), be the function that, given the output of a DNN, returns the classification results in terms of both the index and the node with the greatest value in the output vector. Formally $g(\mathbf{y}) = \langle\argmax(y_1,\dots,y_n), \max(y_1,\dots,y_n)\rangle$.




\paragraph{Abstract DNN's semantics.} Due to the neural network's inherent non-linearity and non-convexity, formal verification typically involves abstracting the input space and evaluating the verification outcome on over-approximated reachable output sets. Let $f^{\sharp}:\wp(\mathbb{R})^m\to\wp(\mathbb{R})^n$ be the abstract semantics of $f$ associating with any $\mathcal{X} \in \wp(\mathbb{R})^m$ a tuple $\mathcal{R} \in \wp(\mathbb{R})^n$.
In detail, $\mathcal{X}$ represents a tuple of possible input sets in $\wp(\mathbb{R})^m$, derived from input perturbations defined by a generic $\ell_p$-norm. Please note the difference between  $\wp(\mathbb{R})^m$ and $\wp(\mathbb{R}^m)$. The former considers a tuple of $m$ independent intervals in $\mathbb{R}$, e.g., $\langle [0,1], [0,1] \rangle$. In contrast, the latter considers a set of subsets in $\mathbb{R}^m$, e.g., $[0,1] \times [0,1] \subseteq \mathbb{R}^2 $, where each element in this subset is a tuple. In this work, we consider the $\ell_\infty$-norm for input perturbations in $\mathcal{X}$, but the same principles also apply to other norms.
Analogously, $\mathcal{R}$ represents the tuple of output reachable sets in $\wp(\mathbb{R})^n$, i.e., the possible subset of outputs achievable when passing $\mathcal{X}$ through $f^{\sharp}$. Without loss of generality, in this paper, we consider hyperrectangles to encode $\mathcal{X}$, i.e., the standard geometric definition as a generalization of a rectangle to multiple dimensions used in abstract interpretation and verification venues. In detail, $\mathcal{X}\in \sset{ \tuple{X_1,\ldots,X_m}}{X_i=[l_i, u_i] \subseteq \mathbb{R}, l_i \leq u_i \;\forall i \in [1,m]}$. Hence, the result of the propagation of $\mathcal{X}$ through $f^{\sharp}$ is $\mathcal{R}\in \sset{ \tuple{Y_1,\ldots,Y_n}}{Y_i= [l_i', u_i'] \subseteq \mathbb{R}, l_i' \leq u_i' \;\forall i \in [1, n]}$. 
We define $\res \defi [1,n] \times \wp(\mathbb{R})$.
Similar to the concrete semantics, we define a function $g^{\sharp}: \wp(\mathbb{R})^n \to \wp(\res)$, which returns the set of both the indices and values of the maximum intervals in a tuple of reachable sets $ \mathcal{R} $. The fact that $ g^{\sharp} $ can return a {\em set} of intervals is due to the non-linearity of the DNN, causing an over-approximation of the reachable output sets and, therefore, leading to potential overlapping. 
Specifically, if $\mathcal{R} = \tuple{ [l_1', u_1'], [l_2', u_2'], \ldots, [l_n', u_n'] }$, then $g^{\sharp}(\mathcal{R})$ returns the set of indices and values of the intervals for which $\max([l_i', u_i'])$ is maximal. For instance, if $\mathcal{R} = \tuple{ [0, 0.8], [0.12, 0.5], [1, 4] }$, then $g^{\sharp}(\mathcal{R})=\{ \tuple{3, [1, 4]} \}$, but if $\mathcal{R} = \tuple{ [2, 5], [0.12, 0.5], [1, 4] }$, then $g^{\sharp}(\mathcal{R}) = \{\tuple{1, [2, 5]}, \tuple{3, [1, 4]}\}$. In Appendix \ref{apx:example_g_sharp}, we report further details of the $g^\#$ procedure.
\clearpage
\paragraph{DNN-Verification problem.} 
\begin{definition}[\textsc{DNN-Verification} Problem]
\label{def:decision_problem}

\phantom{a}

    {\bf Input}: A tuple $\mathcal{T}=\tuple{f^{\sharp},g^\sharp, \mathcal{X}, \mathcal{S}}$.

    {\bf Output}: $\safe \iff g^{\sharp} \circ f^\sharp (\mathcal{X}) \subseteq \mathcal{S}$
\end{definition}

A reachability-based FV tool for neural networks \cite{LiuSurvey} using the abstract DNN semantics takes as input a tuple $\mathcal{T} = \tuple{f^{\sharp}, g^{\sharp}, \mathcal{X}, \mathcal{S}}$, with $\mathcal{X}\in\wp(\mathbb{R})^m$ and $\mathcal{S}\subseteq \res$ characterizing the input-output relationship of the DNN we aim to verify. In this case, $\mathcal{S}$ represents the safe set where the output reachable set should fall for all possible $\mathbf{x} \in \mathcal{X}$. 

\begin{wrapfigure}{r}{0.5\textwidth}
    \centering
    \includegraphics[width=0.48\textwidth]{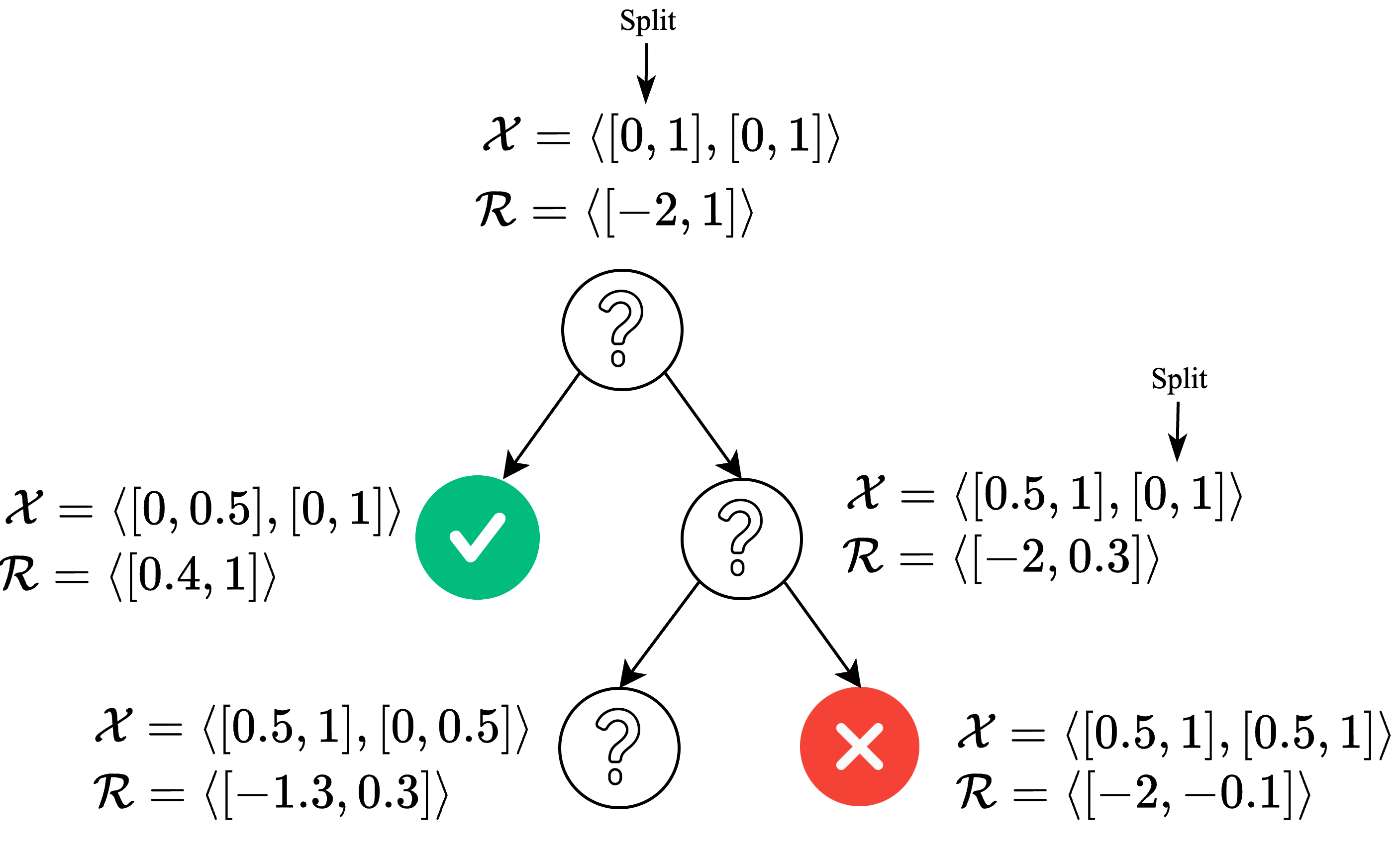}
    \caption{Example of Branch-and-Bound verification process with iterative input refinement approach. In the example we consider the safe set $\mathcal{S} = [0, \infty)$ and, for simplicity, a DNN with a single output node.}
    \label{fig:BaB}
\end{wrapfigure}
By propagating $\mathcal{X}$ through $f^{\sharp}$ and applying $g^{\sharp}$, the tool computes the greatest output reachable set, i.e., the reachable set that has a lower bound greater than all the other upper bounds. It then checks whether it is contained in the desired reachable set, which can be represented either by a safe set $\mathcal{S}$, as in our settings, or an unsafe one. As previously mentioned, the results of $g^{\sharp}$ can produce overlapping reachable sets, thus preventing the possibility of directly stating whether the DNN is safe. To mitigate this issue, FV tools typically split either the elements in the tuple $\mathcal{X}$ into sub-domains \cite{reluval} (each node in Fig.~\ref{fig:BaB}) or unstable neurons in the neural network \cite{bab,bcrown}. When at least an output reachable set of a specific subdomain is not included in $\mathcal{S}$ (the red leaf), the iterative procedure ends---the property is violated since at least one portion of the original domain $\mathcal{X}$ produces a reachable set not included in the desired $\mathcal{S}$. Otherwise, if the tool is sound and complete, it should verify all the leaves before stating that the DNN is provable safe. However, in the worst case, the NP-complete nature of \textsc{DNN-Verification} \cite{Reluplex}, leads to an exponential increase in the number of iterative refinements required to obtain a precise answer, and common incomplete verifiers stop at a given depth returning an \textit{unknown} answer. In this scenario, our approach introduces a potential solution where the safety hierarchy allows in each leaf node the FV tool to provide multiple levels of safety assessments instead of the binary \textit{safe}-\textit{unsafe} classification. This can reduce the number of iterative refinements needed for complete verifiers while it could prevent an \textit{unknown} final outcome for incomplete verifiers.

\section{Related Work}

Existing DNN formal verification tools like AI2\cite{Ai2}, DeepPoly\cite{singh2019abstract}, leverage abstract interpretation to over-approximate the network’s behavior by propagating abstract domains (e.g., intervals, zonotopes) through its layers. This enables efficient verification of properties such as robustness by bounding the output set while maintaining soundness. Another class of tools, such as CROWN, $\alpha$-CROWN and $\beta$-CROWN\cite{crown,acrown,bcrown}, use \textit{linear relaxation} techniques to derive tight linear bounds on outputs by relaxing non-linear activation functions, balancing precision and scalability. These approaches are typically combined with Branch-and-Bound (BaB) \cite{bab} methods, like MN-BaB\cite{MN-BaB}, that divide the original verification problem into smaller subdomains either, for instance, dividing the input perturbation region \cite{reluval} or splitting ReLU neurons into positive/negative linear domains \cite{bab,babSplitRelu}, thus combining precision with completeness at the cost of higher computational effort. Nonetheless, all the state-of-the-art approaches typically produce binary outcomes—either a property holds (safe) or does not hold (unsafe)—and may return \textit{"unknown"} results (for incomplete verifiers) when the abstraction is too coarse.

A similar intuition for reasoning with a hierarchical structure of tolerable unsafe outputs to provide richer analyses has been recently proposed in \cite{giacobazzi2024adversities}, where the concept of \textit{Coherence} as \textit{“weakened robustness properties”} is introduced. 

\begin{definition}[\textit{Coherence} \cite{giacobazzi2024adversities}.]\label{def:coherence}
     Let $\Im: \mathbb{R}^m \to \wp(\mathbb{R}^m)$ be an input perturbation (extensive) function on $\mathbb{R}^m$ (i.e., $\forall\mathbf{x}\in\mathbb{R}^m.\:\mathbf{x}\in\Im(\mathbf{x})$), written $\Im\in\opert(\mathbb{R}^m)$ and $g\comp f:\mathbb{R}^m\rightarrow\mathbb{O}$.
    Let $\mathcal{C}\in\uco(\wp(\mathbb{O}))$\footnote{A function $f:\mathtt{D}\to\mathtt{D}$ is an upper closure operator on $\mathtt{D}$ ($f\in\uco(\mathtt{D})$), if it is idempotent ($\forall x\in\mathtt{D}.\:f\comp f(x)=f(x)$), extensive ($\forall x\in\mathtt{D}.\:f(x)\geq x$), and monotone.} be a domain modeling the categories of classes that can be misclassified.\footnote{For instance in Fig.~\ref{fig:overview} where we have the category "Critical safety sign" modeling the fact that we accept a misclassification between the signs in this category/set} A DNN classifier $g\comp f$ satisfies Coherence w.r.t.\ the input $\Omega\subseteq\mathbb{R}^m$, $\mathcal{C}$ and $\Im$ if $\forall\mathbf
    x\in\Omega.\:\mathcal{C}\comp(g\comp f)\comp\Im(\mathbf{x})\subsetneq\mathbb{O}$, where $g\comp f$ is additively lifted to sets of inputs, those in $\Im(\mathbf{x})$.
\end{definition}

Nonetheless, in \cite{giacobazzi2024adversities}, the authors only focus on verifying whether the {\em concrete} execution of a DNN over a region of inputs (i.e., lifted to sets of concrete inputs in $\mathbb{R}^m$) yields an acceptable answer (i.e., one that is not \textit{“unknown”}) up to a fixed desired output observation, modeled as an abstraction of potential concrete outputs in $\mathbb{O}$. 
In contrast, our approach seeks to handle any {\em over-approximating} execution of the DNN. Specifically, we aim to perturb the abstract input by considering a broader class of abstract inputs rather than (sub)sets of such inputs as defined in Def.~\ref{def:coherence}. Consequently, we start by reformulating the definition using a new, more general concept of input perturbation. We then formally show why our new formulation is necessary and how the reasoning with abstract semantics addresses the limitations of the current definition.
\section{The Abstract DNN-Verification Problem}
This section illustrates how our novel formulation fills the gaps of \cite{giacobazzi2024adversities},  providing a clear motivation for the contribution of this work.
Hence, we start by defining a different notion of input perturbation, widening each element of the input tuple.
This input perturbation is especially useful for reasoning on the DNN's abstract semantics defined in Sec. \ref{sec:preliminaries}.

\begin{definition}[\textit{Widening Input Perturbation function}]
    A {\em widening} input perturbation function $\wIm\in\pert(\wp(\mathbb{R})^m)$ is a function $\wIm:\wp(\mathbb{R})^m \to \wp(\mathbb{R})^m$ such that $\forall \mathcal{X} \in \wp(\mathbb{R})^m.\:  \mathcal{X}\dot{\subseteq} \wIm(\mathcal{X})$, where $\dot{\subseteq}$ is the pointwise extension of $\subseteq$ to tuples of sets. 
\end{definition}
Intuitively, the function $\wIm$ introduces an expansion perturbation, which enlarges the range of intervals to account for more tolerant inputs.


\begin{proposition}\label{prop:interval_abs}
    Given $\mathcal{X}\in\wp(\mathbb{R})^m$ and $\wIm\in\pert(\wp(\mathbb{R})^m)$, then there exists a perturbation $\Im\in\opert(\mathbb{R}^m)$ such that, for any $\mathbf{x}\in\mathcal{X}$, if $\mathbf{z}\in\Im(\mathbf{x})$ then $\mathbf{z}\:\dot{\in}\:\wIm(\mathcal{X})$. Formally, $\forall\mathbf{x}\in\mathbb{R}^m$, $\Im(\mathbf{x})\defi\sset{\mathbf{z}}{\mathbf{z}\:\dot{\in}\:\wIm(\mathcal{X}))}$, where here $\dot{\in}$ is the point-wise $\in$ between tuples, i.e., $\tuple{x_1,\ldots,x_m}\:\dot{\in}\:\tuple{X_1,\ldots,X_m}$ iff $\forall i\in[1,m].\:x_i\in X_i$.
\end{proposition}

Now, let us define the novel notion of {\em Abstract Coherence}, extending the one defined in \cite{giacobazzi2024adversities} to abstract executions of a DNN w.r.t. a given $g^\sharp$.

\begin{definition}[\textit{Abstract Coherence}]\label{def:abscohe}
    Let $f$ be a DNN, and $f^\sharp$ its abstraction. Consider $\mathcal{C}\in\uco(\wp(\res))$,  and $\wIm\in\pert(\wp(\mathbb{R})^m)$. $f$ satisfies Abstract Coherence w.r.t.\ $\tuple{f^\sharp,\:\wIm,\:\mathcal{C}}$ (on $\mathcal{X}\in\wp(\mathbb{R})^m$) if $\mathcal{C}\comp(g^\sharp\circ f^\sharp)\comp\wIm(\mathcal{X})\subsetneq\res$.
\end{definition}

Broadly speaking, $\mathcal{C}$ is the function that abstracts the result of the abstracted DNN, potentially applied to input perturbations, determining whether the results include an {\em acceptable} degree of error (being strictly contained in $\res$), or lose too much (reaching $\res$). 
The following formal result proves our formulation's generality (and novelty) with respect to the one of \cite{giacobazzi2024adversities}.

\begin{theorem}\label{thm:abs_coherence}
\textit{Abstract Coherence} (on $\mathcal{X}\in\wp(\mathbb{R})^m$) implies \textit{Coherence} on any $\mathbf{x}\in\mathcal{X}$, but not vice-versa.
\end{theorem}

\begin{proof}
Let $f^\sharp$ be the abstraction of a DNN $f$ and $g^\sharp$ the decision function on the abstracted results. Consider $\wIm\in\pert(\wp(\mathbb{R})^m)$ and $\mathcal{C}\in\uco(\wp(\res))$. Suppose to have \textit{Abstract Coherence} of the DNN $g\comp f$ (w.r.t.\ $\tuple{f^\sharp,\:\wIm,\:\mathcal{C}}$). 
From Prop.~\ref{prop:interval_abs} we have that 
for any $\mathbf{x}\;\dot{\in}\;\mathcal{X}$,  $\mathbf{z}\in\Im(\mathbf{x})$ implies $\mathbf{z}\:\dot{\in}\:\wIm(\mathcal{X})$. 
Moreover, $\forall\mathcal{Z}\in\wp(\mathbb{R})^m$ we have $g\comp f(\mathcal{Z})\defi\sset{g\comp f(\mathbf{z})}{\mathbf{z}\:\dot{\in}\:\mathcal{Z}}\subseteq g^\sharp\comp f^\sharp(\mathcal{Z})$. Suppose now $\mathcal{C}\comp(g^\sharp\comp f^\sharp)\comp\wIm(\mathcal{X})\subsetneq\res$, then  by monotonicity of $\mathcal{C}$, $\forall\mathbf{x}\;\dot{\in}\;\mathcal{X}$ we have $\mathcal{C}\comp(g\comp f)\comp\Im(\mathbf{x})\subseteq\mathcal{C}\comp(g\comp f)(\wIm(\mathcal{X}))\subseteq \mathcal{C}\comp(g^\sharp\comp f^\sharp)(\wIm(\mathcal{X}))\subsetneq\res$. Hence, \textit{Abstract Coherence} w.r.t.\ $\tuple{f^\sharp,\:\wIm,\:\mathcal{C}}$ on $\mathcal{X}$ implies \textit{Coherence} w.r.t.\ $\tuple{\Im,\mathcal{C}}$ on any $\mathbf{x}\;\dot{\in}\;\mathcal{X}$.\\
Let us provide the intuition that the opposite implication, in general, does not hold. Consider $\mathcal{X}\subseteq \mathbb{R}^m$, and suppose $\forall\mathbf{x}\in\mathcal{X}.\:\mathcal{C}\comp(g\comp f)\comp\Im(\mathbf{x})\subsetneq\res$  (\textit{Coherence}). Let us define $\wIm(\mathcal{X})\defi\tuple{X_1,\ldots,X_m}\in\wp(\mathbb{R})^m$ such that $X_i=\sset{x_i}{\tuple{x_1,\ldots,x_m}\in\Im(\mathbf{x}),\ \mathbf{x}\;\dot{\in}\;\mathcal{X}}$.
Then $\mathcal{C}\comp(g\comp f)\comp\Im(\mathbf{x})\subseteq\mathcal{C}\comp(g\comp f)\comp\wIm(\mathcal{X})\subseteq\mathcal{C}\comp(g^\sharp\comp f^\sharp)\comp\wIm(\mathcal{X})$. Therefore we can always find a DNN approximation $f^\sharp$, a class abstraction $\mathcal{C}$ and an input perturbation $\Im$ such that $\mathcal{C}\comp(g\comp f)\comp\Im(\mathbf{x})\subsetneq\res$ does not imply $\mathcal{C}\comp(g^\sharp\comp f^\sharp)\comp\wIm(\mathcal{X})\subsetneq\res$.
\end{proof}

For completeness, we report in Appendix \ref{apx:example_proof} a concrete example where \textit{Coherence} does not imply \textit{Abstract Coherence}.

At this point, we can define the \textsc{Abstract DNN-Verification} ({\sc ADV}) problem as an instantiation of Def.~\ref{def:abscohe}, verifying whether the result of the \textsc{DNN-Verification} applied to the perturbation of all set of abstracted inputs $\wIm(\mathcal{X})$ still respect the safe (or at least not unsafe) set determined by the abstraction $\mathcal{C}$. 

\begin{tcolorbox}
  \vspace{-0.2cm}
  \begin{definition}[\sc Abstract DNN-Verification]\label{def:ADV}
  
\vspace{2mm}

\phantom{a}

\vspace{2mm}

    \hspace{3.5mm}{\bf Input}: A tuple $\mathcal{T}=\langle f^{\sharp}, g^{\sharp},\wIm(\mathcal{X}), \mathcal{C}\rangle$ 

    \vspace{1mm}

    \hspace{3.5mm}{\bf Output}: \textit{Abstract Safe} $\Leftrightarrow 
    \mathcal{C}\comp(g^\sharp\comp f^{\sharp})\comp\wIm(\mathcal{X})\subsetneq \res$

  \vspace{-0.1cm}
 \end{definition}
\end{tcolorbox}

Where we recall that $\res$ here represents an undesired, unsafe potential result in the class prediction. 

With this formalization, the \textsc{DNN-Verification} problem becomes an instantiation of the abstract one, with $\wIm=\id$ the identity function, and $\mathcal{C}(g^\sharp\comp f^{\sharp}(\mathcal{X})) = \mathcal{S}\subsetneq\res$\footnote{If $\mathcal{S}=\res$ the problem is meaningless, becoming always safe.} if $g^\sharp\comp f^{\sharp}(\mathcal{X})\subseteq\mathcal{S}$, $\mathcal{C}(g^\sharp\comp f^{\sharp}(\mathcal{X}))=\res$, unsafe, otherwise. Hence, we have the following result.

\begin{proposition}\label{prop:safe_implication}
    If a DNN is safe w.r.t.\ the safe set $\mathcal{S}\subsetneq\res$, then it is also abstract safe w.r.t.\ any abstraction $\mathcal{C}\in\uco(\wp(\res))$ such that $\mathcal{S}\in\mathcal{C}$.
\end{proposition}

Hence, the design of $\mathcal{C}$ defines the tolerable level of misclassification. To clarify this, we provide an illustrative example.
\begin{figure}[h!]
	\begin{center}
		\scalebox{0.55} {
			\def\layersep{3cm}
			\begin{tikzpicture}[shorten >=1pt,->,draw=black!50, node
				distance=\layersep,font=\footnotesize]
				
				\node[input neuron] (I-1) at (0,-1) {$x_1$};
				\node[input neuron] (I-2) at (0,-2.5) {$x_2$};
                \node[input neuron] (I-3) at (0,-4) {$x_3$};

				\node[left=-0.05cm of I-1] (b1) {$X_1=[0, 1]$};
				\node[left=-0.05cm of I-2] (b2) {$X_2=[0, 1]$};
                \node[left=-0.05cm of I-3] (b3) {$X_3=[0.8, 1]$};
				
				\node[hidden neuron] (H-1) at (1.2*\layersep,-1) {$0$};
				\node[hidden neuron] (H-2) at (1.2*\layersep,-2.5) {$0$};
                \node[hidden neuron] (H-5) at (1.2*\layersep,-4) {$0$};
				
				\node[hidden neuron] (H-3) at (1.8*\layersep,-1) {$h_1^a$};
				\node[hidden neuron] (H-4) at (1.8*\layersep,-2.5) {$h_2^a$};
                \node[hidden neuron] (H-6) at (1.8*\layersep,-4) {$h_3^a$};
				
			    \node[output neuron] at (3*\layersep, 0.5) (O-0) {$3.8$};
                \node[output neuron] at (3*\layersep, -1) (O-1) {$3.5$};
                \node[output neuron] at (3*\layersep, -2.5) (O-2) {$3.4$};
                \node[output neuron] at (3*\layersep, -4) (O-3) {$-2$};
                \node[output neuron] at (3*\layersep, -5.5) (O-4) {$3.2$};

                \node[below=0.05cm of O-0]{$y_1$};
                \node[below=0.05cm of O-1]{$y_2$};
                \node[below=0.05cm of O-2]{$y_3$};
                \node[below=0.05cm of O-3]{$y_4$};
                \node[below=0.05cm of O-4]{$y_5$};

                \node[right=0.05cm of O-0] (b0) {$Y_1=[2.3, 7.4]$};
                \node[right=0.05cm of O-1] (b1) {$Y_2=[7.8, 14.7]$};
                \node[right=0.05cm of O-2] (b2) {$Y_3=[5.4, 6.62]$};
                \node[right=0.05cm of O-3] (b1) {$Y_4=[-11, -5.2]$};
                \node[right=0.05cm of O-4] (b2) {$Y_5=[7.8, 17.4]$};

				
				\draw[nnedge] (I-1) --node[above, pos=0.25] {$4$} (H-1);
				\draw[nnedge] (I-1) --node[above, pos=0.25] {$-2$} (H-2);
                \draw[nnedge] (I-1) --node[above, pos=0.25] {$1$} (H-5);
				\draw[nnedge] (I-2) --node[above, pos=0.1] {$-1$} (H-1);
                \draw[nnedge] (I-2) --node[above, pos=0.25] {$3$} (H-2);
				\draw[nnedge] (I-2) --node[below, pos=0.1] {$0$} (H-5);
                \draw[nnedge] (I-3) --node[below, pos=0.45] {$2$} (H-1);
                \draw[nnedge] (I-3) --node[below, pos=0.3] {$-1$} (H-2);
				\draw[nnedge] (I-3) --node[below] {$5$} (H-5);
				
				\draw[nnedge] (H-1) --node[above] {ReLU} (H-3);
				\draw[nnedge] (H-2) --node[below] {ReLU} (H-4);
                \draw[nnedge] (H-5) --node[below] {ReLU} (H-6);

                \draw[nnedge] (H-3) --node[above, pos=0.8] {$0.5$} (O-0);
				\draw[nnedge] (H-3) --node[above, pos=0.88] {$0.5$} (O-1);
                \draw[nnedge] (H-3) --node[above, pos=0.85] {$0$} (O-2);
                \draw[nnedge] (H-3) --node[above, pos=0.85] {$-0.5$} (O-3);
                \draw[nnedge] (H-3) --node[above, pos=0.85] {$1$} (O-4);

                \draw[nnedge] (H-4) --node[above, pos=0.8] {$-1$} (O-0);
				\draw[nnedge] (H-4) --node[above, pos=0.78] {$1$} (O-1);
                \draw[nnedge] (H-4) --node[above, pos=0.78] {$0.1$} (O-2);
                \draw[nnedge] (H-4) --node[above, pos=0.77] {$0.5$} (O-3);
                \draw[nnedge] (H-4) --node[above, pos=0.75] {$1$} (O-4);

                \draw[nnedge] (H-6) --node[above, pos=0.8] {$0.1$} (O-0);
				\draw[nnedge] (H-6) --node[above, pos=0.78] {$1$} (O-1);
                \draw[nnedge] (H-6) --node[above, pos=0.78] {$0.5$} (O-2);
                \draw[nnedge] (H-6) --node[above, pos=0.75] {$-1$} (O-3);
                \draw[nnedge] (H-6) --node[above, pos=0.75] {$1$} (O-4);


				\node[below=0.05cm of H-1] (b1) {$[-1, 6]$};
				\node[below=0.05cm of H-2] (b2) {$[-3, 3]$};
                \node[below=0.05cm of H-5] (b3) {$[0, 6]$};

				\node[below=0.05cm of H-3] (b1) {$[0, 6]$};
				\node[below=0.05cm of H-4] (b2) {$[0, 3]$};
                \node[below=0.05cm of H-6] (b3) {$[0, 6]$};


		
				\node[annot,above of=H-1, node distance=0.8cm] (hl1) {Weighted
					sum};
				\node[annot,above of=H-3, node distance=0.8cm] (hl2) {Activated Layer };
				\node[annot,above of=I-1, node distance=0.8cm] {Input };
				\node[annot,above of=O-0, node distance=0.8cm] {Output };
			\end{tikzpicture}
		} 
		
	\end{center}
    \caption{Toy DNN employed in the \textsc{Abstract DNN-Verification} example. $\res\defi\{c_i\defi\tuple{i,Y_i}~|~i\in[1,5]\}$.}
    \label{fig:SingleRechableSet2}
\end{figure}
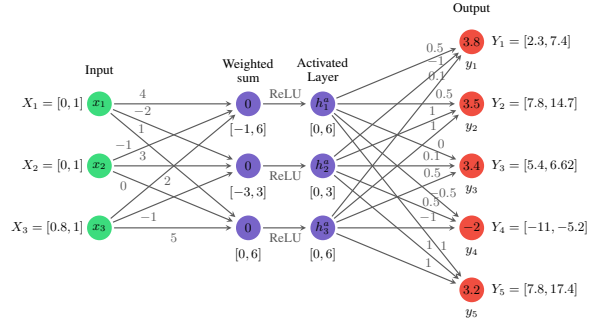
 
\begin{example}\label{ex:gerarchia}
Consider the DNN in Fig.~\ref{fig:SingleRechableSet2}, where the abstracted input is $\mathcal{X}\defi\tuple{[0,1],[0,1]$,$[0.8,1]}$ and suppose $\mathcal{S}=\{c_2, c_3, c_5\}\subsetneq\res$. This means that both $c_1$ and $c_4$ are unsafe, but we can suppose that there is a different degree of unsafeness between them, namely, $c_4$ is completely unsafe, while $c_1$ is only potentially unsafe (for instance, because it is quite unlikely to achieve). Starting from the concrete safe output set, we can define an output abstraction setting a tolerable misclassification, i.e., $\mathcal{C}=\{\res, \{c_1,c_2,c_3,c_5\},\mathcal{S}\}$ (i.e., from  $\mathcal{S}$ we allow a more tolerable {\em unprecise} classification in $\{c_1, c_2, c_3, c_5\}$). We then consider the DNN as {\em abstract safe} if and only if the output abstraction of the network for $\mathcal{X}$ satisfies \textit{Abstract Coherence} w.r.t.\ $\tuple{f^\sharp,g^\sharp,\id,\mathcal{C}}$, namely $\mathcal{C}\comp(g^\sharp\comp f^\sharp)(\mathcal{X})$ is different from $\res$ (here we suppose not to perturb the input, i.e., $\wIm=\id$).
\end{example}

Although lifting the output of $g^{\sharp}$ can yield more informative answers with potentially fewer verification queries, the \textsc{Abstract DNN-Verification} problem, in the worst case, remains difficult to solve. Even with a large abstraction of tolerable answers, the number of iterative refinements required to reach a solution could still grow exponentially with the size of the instance. As a result, we have the following.

\begin{proposition}\label{prop:NPC}
    The \textsc{Abstract DNN-Verification} problem is NP-Complete.
\end{proposition}

In Appendix \ref{apx:hardness}, we fully discuss the hardness of the problem. Nevertheless, abstracting the safety test and allowing the formal verification tool to provide multiple responses, rather than simply classifying the system as \textit{safe} or \textit{unsafe}, can potentially reduce the cost of the iterative refinement process, improving scalability in practice. We will demonstrate this aspect in practice through our empirical evaluation in Sec. \ref{sec:empirical_eval}.

\subsection{Applying \textsc{ADV}: a Simple Example}

Let us consider again the DNN in Fig.~\ref{fig:SingleRechableSet2}, the abstracted input $\mathcal{X}$ and the safety set $\mathcal{S}=\{c_2, c_3, c_5\}\subsetneq\res$ as in Example~\ref{ex:gerarchia}.
The propagation of $\mathcal{X}$ through $f^{\sharp}$ produces the abstracted DNN output $\mathcal{R} = \tuple{[2.3, 7.4],[7.8, 14.7],[5.4, 6.62], [-11, -5.2], [7.8, 17.4]}$. In this situation, $g^\sharp(\mathcal{R})=\{\tuple{2,Y_2},\tuple{5,Y_5}\}=\{c_2,c_5\}\subseteq\mathcal{S}$, where both $c_2$ and $c_5$ are provable safe. This corresponds to \textit{Abstract Coherence} with $\mathcal{C}_s\defi\{\res,\mathcal{S}\}$.
This result tells us that we do not know which of the outputs will be chosen between $c_2$ and $c_5$, but we know that it will be one of these two, and both are safe. And, formally, we have $\mathcal{C}_s(g^{\sharp}(\mathcal{R})) = \mathcal{C}_s(\{c_2,c_5\})= \{c_2, c_3, c_5\}\subsetneq\res$, meaning that the DNN is provable safe, and thus for Proposition \ref{prop:safe_implication} also \textit{abstract} safe.\\
By applying the input perturbation $\wIm$, such as $\varepsilon = 0.8$, on one input feature at a time, we can analyze how variations in individual inputs influence the output's abstract results, allowing us to gain additional safety insights into the DNN's behavior.
In detail, let us consider $\wIm(\mathcal{X}) \defi \langle[0,1], [0,1], [0,1]\rangle$, i.e., we perturb the last input feature's lower bound by $\varepsilon$. Hence, we have that $f^{\sharp}(\wIm(\mathcal{X}))$ produces $\mathcal{R} = \tuple{[0.8, 7.4],[3.5, 15.5], [3.4, 6.7],[-11, -0.5], [3.2, 18.2]}$. By applying $g^\sharp$, which collects all the overlapping outputs, we obtain as set of potentially maximal values $\{c_1, c_2, c_3, c_5\}$, telling us that now the output $c_1$ (only potentially safe) is such that $Y_1$ overlaps all the intervals of safe outputs $\{c_2, c_3, c_5\}$, i.e., $g^\sharp(\mathcal{R})=\{c_1, c_2, c_3, c_5\}\not\subseteq\mathcal{S}$ (and indeed $\mathcal{C}_s(g^{\sharp}(\mathcal{R}))=\res$). This answer tells us that, by perturbing the $X_3$ feature, in any case, it is mathematically proven that the $c_4$ output cannot be selected, but maybe $c_1$ can, making the DNN no more provably safe. Nevertheless, if we consider instead abstract safety w.r.t.\ $\mathcal{C}=\{\res, \{c_1,c_2,c_3,c_5\},\mathcal{S}\}$ as in Example~\ref{ex:gerarchia}, then we can prove abstract safety without the necessity of further investigation, since $\mathcal{C}(g^{\sharp}(\mathcal{R}))=\{c_1, c_2,c_3,c_5\} \subsetneq \res$.

To confirm this result, we employ a recent off-the-shelf enumeration strategy \cite{eProve} to identify all regions of the state space where the model satisfies specified requirements. This verification process involves pruning the original input domain to enumerate all the regions where the model satisfies the output abstraction $\mathcal{C}$. Specifically, given the result of $g^\sharp$, we know that in $\mathcal{X}$ the DNN will always select an output in $\{c_1, c_2, c_3, c_5\}$. Hence, the enumeration allows us to understand which part of the input space makes the DNN only \textit{abstract safe} and not provable safe. Although solving this problem in an exact fashion is, in general, computationally prohibitive due to the \#P-hardness of the problem \cite{CountingProVe}, in this toy example, we successfully identify the exact portions of the original input domain $\wIm(\mathcal{X})$  where the model selects $c_1$. Our analysis reported in Fig.~\ref{fig:enum} confirms that only the abstraction $\mathcal{C}$ is respected for the specified perturbed input region. In fact, as we can notice in the figure highlighted in orange, there is a small portion of the perturbed input domain $\tuple{[0,1], [0,1],[0,1]}$ where the model selects the output $c_1$, thus violating $\mathcal{C}_s$.\\ 
Interestingly, by perturbing $X_1$ as done for $X_3$, i.e., by taking $\wIm(\mathcal{X}) \defi\tuple{ [0, 1.8],[0,1],[0.8, 1]}$ the abstract result does not change, i.e., abstract safety w.r.t.\ $\mathcal{C}$ still holds. Hence, we can conclude that a perturbation to either $X_1$ or $X_3$ equally impacts the {\em abstract} safety degree of the DNN.
Finally, if $\wIm(\mathcal{X}) \defi \tuple{[0,1], [0,1.8],[0.8,1]}$ the result is $\mathcal{R}=\tuple{[2.3, 7.4],[7.8, 14.7],[5.4, 6.62],[-11.0, -5.2],[7.8, 17.4]}$, which produces $g^{\sharp}(\mathcal{R}) = \{c_2, c_5\}$, implying that the DNN is abstract safe w.r.t.\ $\mathcal{C}$, but it is also safe (w.r.t.\ $\mathcal{C}_s$). Although this result may not provide interesting safety details, it actually reveals that perturbations on the input feature $X_1$ and $X_3$ have a greater impact on safety degree than perturbations on $X_2$, which preserves \textit{concrete} safety for the same range of perturbation. This simple example shows the expressiveness and potentiality of this new safety formulation, as with a single interval propagation, we can characterize and confine the safety degree of the DNN for a specified safety hierarchy $\mathcal{C}$ designed by the analyst.

\begin{figure}[h!]
  \begin{center}
    \includegraphics[width=0.7\linewidth]{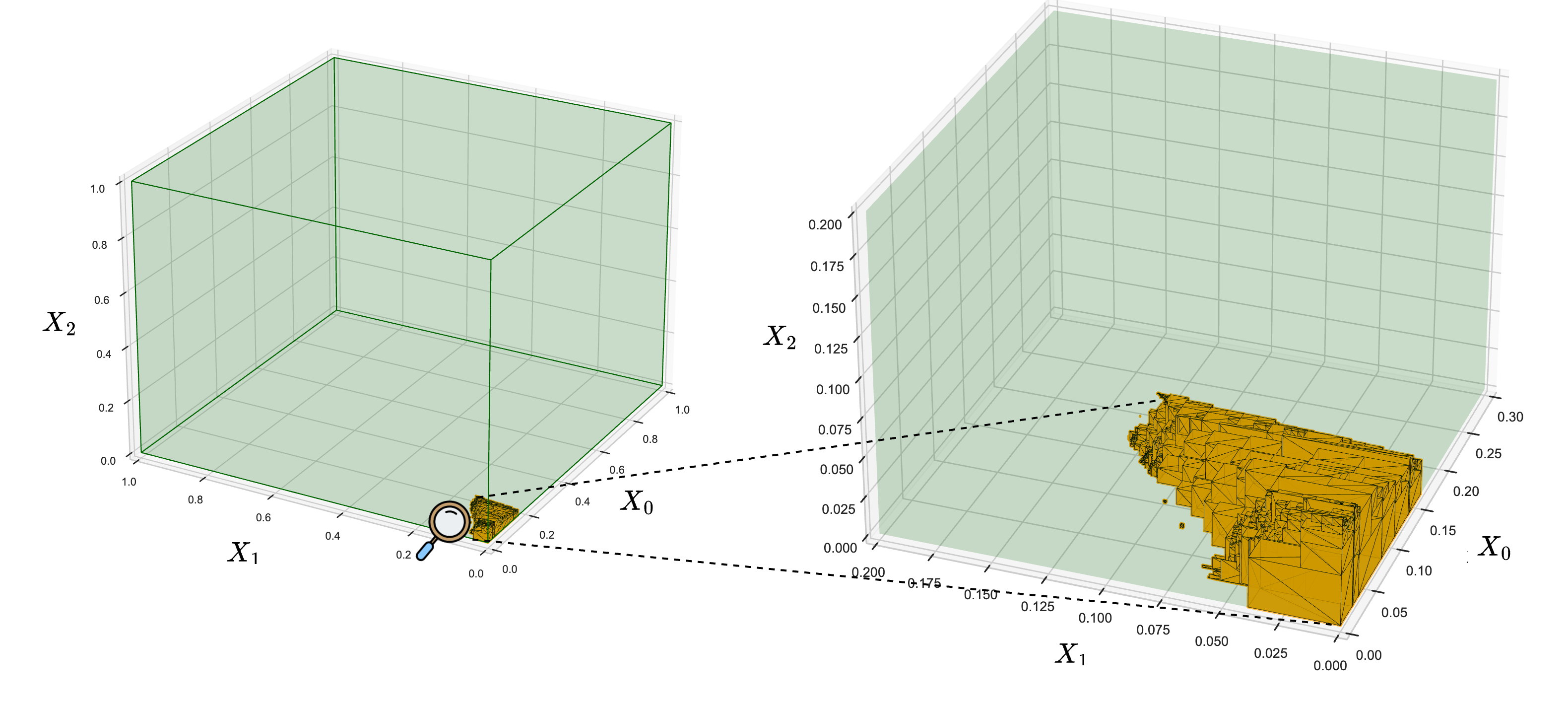} 
      \end{center}
    \caption{Enumeration result of the verification process for $\mathcal{X}$. 
    The green area represents the portion of the input space mapped in the output abstraction $\{c_2,c_3,c_5\}$, while in orange, the one in $\{c_1\}$.}
    \label{fig:enum}
\end{figure}

\section{Empirical Evaluation}\label{sec:empirical_eval}

In this section, we guide the reader in understanding the importance and impact of this novel encoding for the verification problem of deep neural networks. All data are collected on a cluster running Rocky Linux 9.34 equipped with Nvidia RTX A6000 (48 GiB) and a CPU AMD Epyc 7313 (16 cores). The code used to collect the results, and detailed instructions to reproduce our experiments are available in the supplementary material.

We first consider the safety verification of a realistic deep reinforcement learning navigation task, namely \textit{Habitat-Lab} \cite{habitat3,habitat2}. In this experiment, we employ the tool of \cite{eProve}, which allows us to select more realistic adversarial attacks (e.g., light attacks and sensor ruptures) rather than standard $\ell_\infty$-ball perturbation supported in state-of-the-art FV tools like \textit{$\alpha,\beta$-CROWN} \cite{crown,acrown,bcrown}. In detail, we show how our novel formulation, with even relatively simple hierarchies, enables the ranking of adversarial inputs based on their impact, providing valuable insights into model robustness that traditional binary verification cannot capture.
To further assess the impact of our proposal, we consider the robustness verification of well-known classification tasks on state-of-the-art benchmarks such as CIFAR10 \cite{cifar10}, 
employed in VNN-COMP \cite{VNN-comp2023}. In this context, we employ \textit{$\alpha,\beta$-CROWN} \cite{crown,acrown,bcrown}, and we study the impact of different perturbation levels on standard discrete classes and the different output abstractions defined for each domain tested. Notably, the abstract safety specifications proposed in this work can be easily encoded using the standard VNN-LIB format and thus can be verified by any other existing FV tools. 
Our empirical evaluation demonstrates that our novel formulation is already applicable in real-world scenarios, as evidenced by its successful use not only on standard VNN-COMP models but also on more complex architectures, such as a long short-term memory (LSTM) network with a ResNet18 visual backbone deployed in the realistic Habitat 3.0 environment.\\
\paragraph{Habitat-Lab Experiment.} 
Habitat \cite{habitat2,habitat3} is a high-performance simulator designed for training and evaluating deep reinforcement learning models, particularly in 3D environments for embodied AI tasks. In the \textit{Social Navigation} task \cite{habitat2,habitat3} employed in our evaluation, we have a robotic agent that has to follow a humanoid without any collision. This evaluation considers the publicly released trained agent from the Habitat-Lab GitHub repository (\url{https://github.com/facebookresearch/habitat-lab}). 

\begin{figure}[h!]
\centering
\includegraphics[width=.33\textwidth]{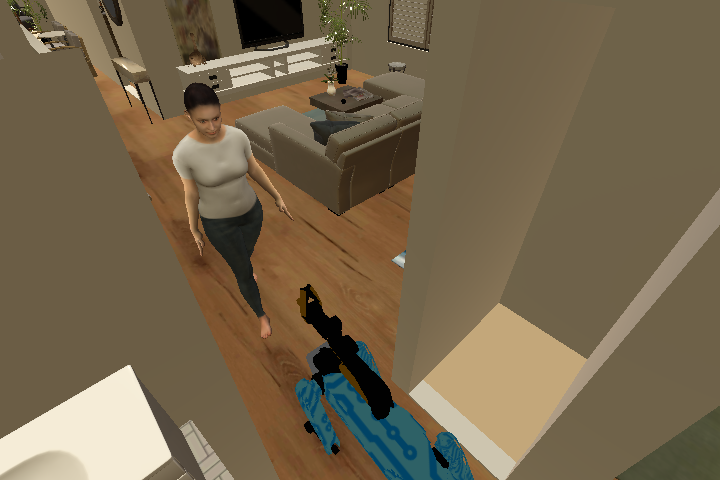}\hfill
\includegraphics[width=.33\textwidth]{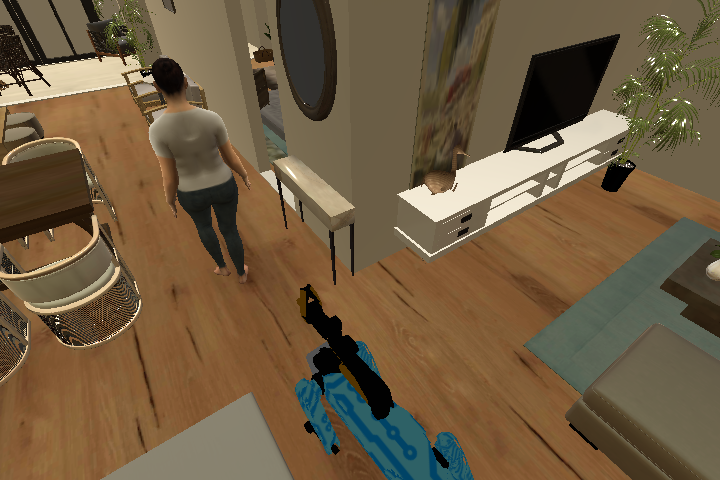}\hfill
\includegraphics[width=.33\textwidth]{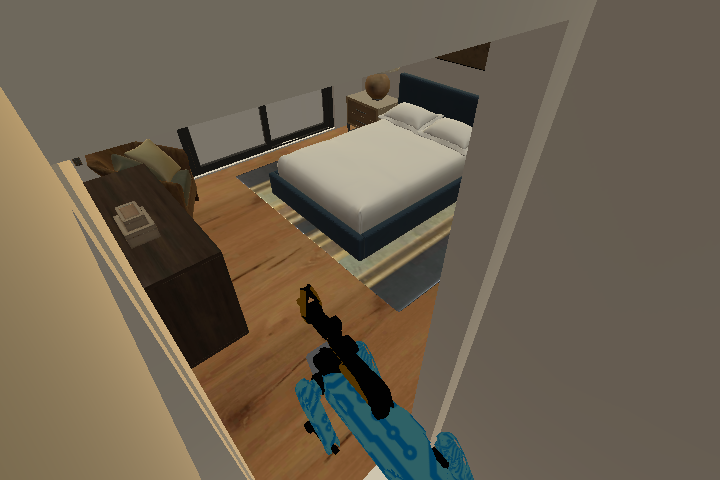}

\caption{Unsafe scenarios from the Habitat Lab experiments: On the left, the humanoid approaches the robot, which should move backward to avoid collision. In the center, the humanoid moves away, and the robot should follow while avoiding obstacles. On the right, the humanoid approaches from behind, and the robot should turn to avoid an unexpected collision while searching for the humanoid.}
\label{fig:habitat_props}
\end{figure}
In detail, we identify three unsafe situations (the first two depicted in Fig.\ref{fig:habitat_props}) from a dataset of trajectories
observed during evaluation, and we defined a set of provably safe actions and potentially safe actions (i.e., tolerable but less preferred than the safe ones), setting different thresholds for the linear and angular velocities that the agent should respect. Hence, the safe output abstraction $\mathcal{C}_1$ here represents more stringent velocity thresholds, while the \textit{abstract safe} output abstractions $\mathcal{C}_2$ consider more relaxed ones. For instance, if the agent in the first unsafe scenario only slowly moves backward and turns around, it can still be tolerable (i.e., abstract safe) but is less preferable than moving backward with a higher linear velocity and no angular velocity, i.e., without turning around. If the agent does not respect any of the output abstractions, it is considered unsafe. For the sake of clarity, in Fig. \ref{fig:hierarchy_habitat}, we only report the abstraction for the first unsafe situation, but a similar strategy is applied to all the other situations considered using different thresholds as described in Tab.~\ref{tab:concrete_classes_habitat} of Appendix \ref{appendix:full_empirical}.

\begin{figure}[h!]
    \centering
    \includegraphics[width=0.7\linewidth]{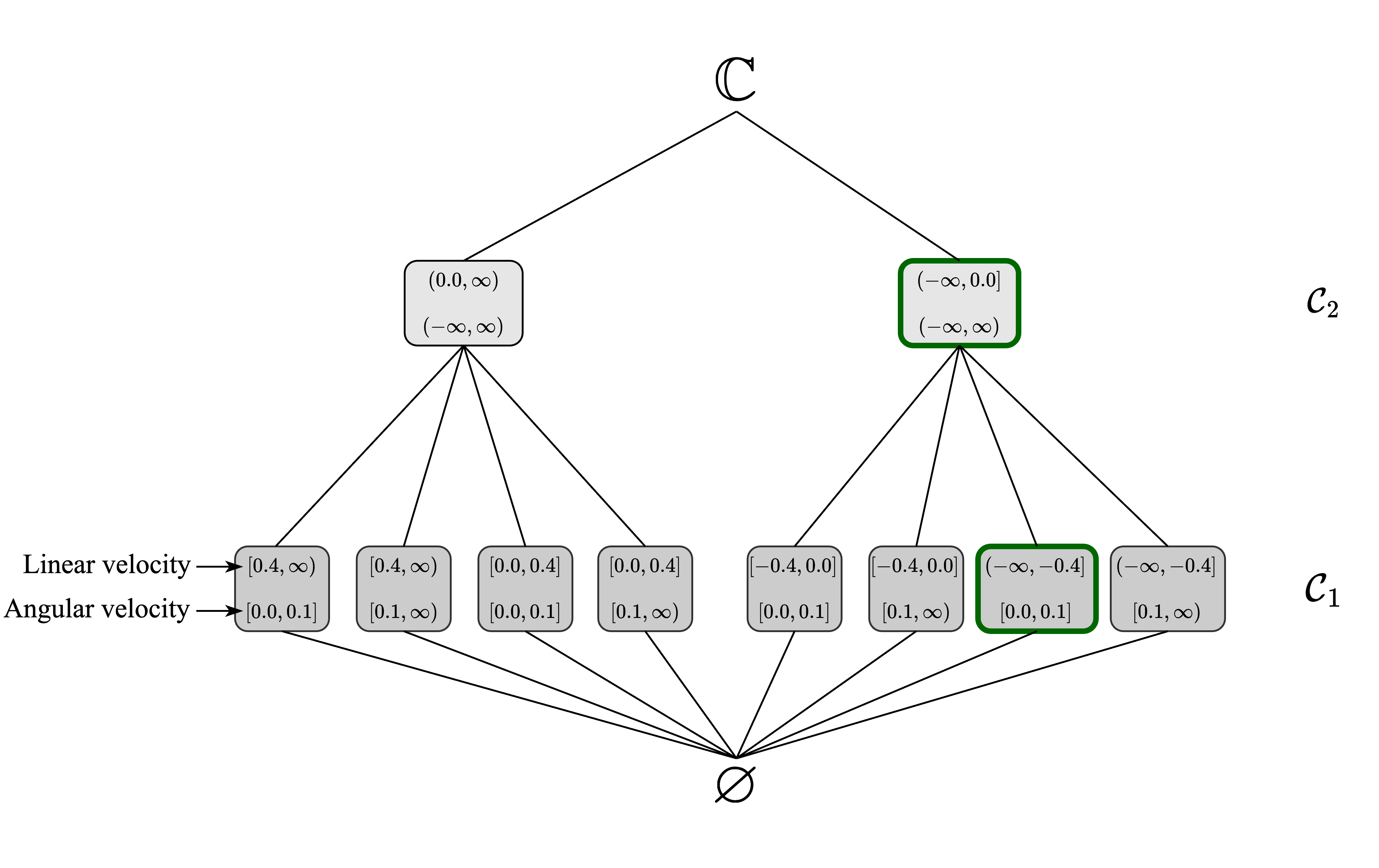}
    \caption{Abstraction hierarchy used in the Habitat Lab experiment for \textit{Situation 1}. For $\mathcal{C}_1$ and $\mathcal{C}_2$ we highlighted in dark green the \textit{safe} and \textit{abstract safe} desired outputs.}
    \label{fig:hierarchy_habitat}
\end{figure}

Our empirical evaluation is based on three different types of adversarial attacks (reported Fig. \ref{fig:habitat_adv}).

 In detail, we consider:
 
\begin{itemize}
    \item \textbf{Light Attacks}: In this attack, inspired by \cite{hsiao2024natural}, a random region of the original image is selected, and a natural light effect is simulated by proportionally increasing the pixel values within the patch, effectively overlaying it up to a complete white patch. In practice, this attack is performed by selecting a patch $ P $ in the image and perturbing its pixels according to the formula:  
    \[
    p_i' = \min\{p_i + \epsilon \cdot (1.0 - p_i), 1.0\}, \quad \forall p_i \in P.
    \]
    where $p_i$ represents the original pixel value, and $p_i'$ is the perturbed value. 
    \item \textbf{$\epsilon$-ball Perturbations}: adopted in all of our experiments, which define $\ell_\infty$ balls with radius $\epsilon$ to simulate Gaussian noise applied over all input features.
    \item \textbf{Sensor Rupture}: This attack is modeled by selecting a patch $P$ in the image and setting each pixel within the patch to black (i.e., a value of 0) with a probability $\epsilon$. This corresponds to restricting the affected pixels to the closed interval $[0, 0]$ during verification. Mathematically, we define this attack as:  
    \[
    p_i' = 
    \begin{cases} 
    0 & \text{with probability } \epsilon, \\
    p_i & \text{with probability } (1 - \epsilon),
    \end{cases}  
    \quad \forall p_i \in P,
    \]  
\end{itemize}

\begin{figure}[h!]
    \centering
    \includegraphics[width=\linewidth]{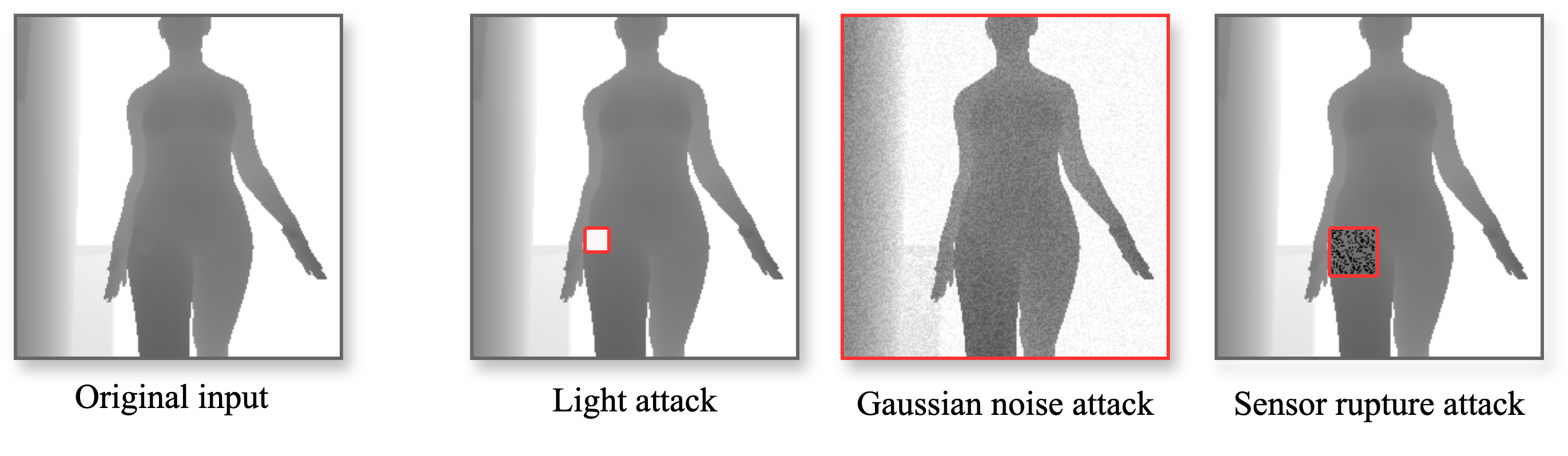}
    \vspace{-0.5cm}
    \caption{A visual representation of each type of attack used in our empirical evaluation applied on the same arm depth camera input from the habitat-lab experiment. On the left is the original input before the adversarial attack. On the right, the different attacks are highlighted in red: the "light patch" attack with patch size $16\times 16$ and $\epsilon$ set to 1.0, the $\ell_\infty$ $\epsilon$-ball attack with $\epsilon$ set to 0.08 applied to all the input features, and finally, the "sensor rupture" attack with patch size $32\times 32$ and $\epsilon$ set to 0.3.}
    \label{fig:habitat_adv}
\end{figure}


Results in Tab.~\ref{tab:result_habitat} demonstrate the varying impact of different attacks across situations. In \textit{Situation 1}, ``light patch" and $\ell_\infty$ attacks exhibit similar effects, maintaining abstract safe behavior under relaxed thresholds for linear and angular velocities, whereas ``sensor rupture" attacks result in unsafe behavior by exceeding the linear velocity threshold, risking a collision. For \textit{Situation 2}, ``light patch" attacks lead to safe behavior under stringent thresholds, while $\ell_\infty$ and ``sensor rupture" attacks produce only abstract safe outcomes by meeting relaxed thresholds. In \textit{Situation 3}, all attacks result in safe behavior, respecting the imposed thresholds. We highlight that our novel problem formulation allows for a more nuanced understanding of the impact of different attacks. For example, in \textit{Situation 1}, standard verification would classify all attacks with the same impact (all producing unsafe results) due to the lack of concrete, safe outcomes. In contrast, our approach reveals that ``light patch" and $\ell_\infty$ perturbations have a lesser impact than ``sensor rupture". This experiment demonstrates how the \textsc{ADV} problem can be leveraged in realistic safety-critical tasks to develop a pipeline for ranking adversarial attacks based on their impact on the output abstractions associated with the unsafe situation verified.

\begin{table}[h!]
    \tiny
    \centering
    \caption{Empirical results on Habitat Lab benchmark.}
    \label{tab:result_habitat}
   \begin{tabular}{ccccccc} 
     \Xhline{3\arrayrulewidth}
      & \vspace{-2mm}& & & & & \\ 
     \textbf{Situation} & \textbf{Perturb.} & \textbf{Patch size} & $\epsilon$  & \textbf{Safe} & \textbf{Abstract safe} & \textbf{Unsafe} \\ [0.5ex] 
     \Xhline{3\arrayrulewidth}
      1 & Light patch & $16\times16$ & $1.0$ &  & \ding{51} &  \\ 
     
      2 & Light patch & $16\times16$ & $1.0$ &  \ding{51} &  &  \\
    
      3 & Light patch & $16\times16$ & $1.0$ &  \ding{51} &  &  \\
     \Xhline{3\arrayrulewidth}
       1 & $\ell_{\infty}$ & - & $0.08$  &  & \ding{51} &  \\ 
     
      2 & $\ell_{\infty}$ & - & $0.08$ &  & \ding{51} &  \\
    
      3 & $\ell_{\infty}$ & - & $0.08$ &  \ding{51} &  &  \\
     \Xhline{3\arrayrulewidth}
     1 & Sensor rupture & $32\times32$ & $0.3$ &  &  & \ding{51} \\ 
     
      2 & Sensor rupture & $32\times32$ & $0.3$ &  & \ding{51} & \\
    
      3 & Sensor rupture & $32\times32$ & $0.3$ &  \ding{51} &  &  \\
    \Xhline{3\arrayrulewidth}
    \end{tabular}
\end{table}

\paragraph{CIFAR10 Experiment.} In this experiment, we test the effectiveness of the \textsc{ADV} to understand the impact of different $\ell_\infty$ $\epsilon$-ball perturbations on different output abstractions. Specifically, we define three distinct degrees of output abstractions to evaluate robustness to perturbations, which we report in Fig.~\ref{fig:hierarchy_CIFAR}. The first output abstraction corresponds to the precise classification modeled by $\mathcal{C}_1\defi\{\res,\{\mbox{Airplane}\}, \{\mbox{Car}\}, \{\mbox{Bird}\}, \{\mbox{Cat}\},\{\mbox{Deer}\}, \{\mbox{Dog}\},$\\$ \{\mbox{Frog}\}, \{\mbox{Horse}\}, \{\mbox{Ship}\}, \{\mbox{Truck}\},\varnothing\}$, where we do not allow any misclassification. Then, we start by aggregating some elements to create a set of categories, modeling some degree of acceptable error in the classification. With this purpose, we define $\mathcal{C}_2 \defi\{\res,\hfill\mbox{Other Vehicles},\hfill \mbox{Terrain Vehicles},\hfill \mbox{Mammiferous},\ $\\$  \mbox{non-Mammiferous},\{\mbox{Horse}\},\varnothing\}$ where the order relation depicted in Fig.~\ref{fig:hierarchy_CIFAR} determines the set composition of the given categories.

\begin{figure}[h!]
    \vspace{-3mm}
    \centering
    \includegraphics[width=0.65\linewidth]{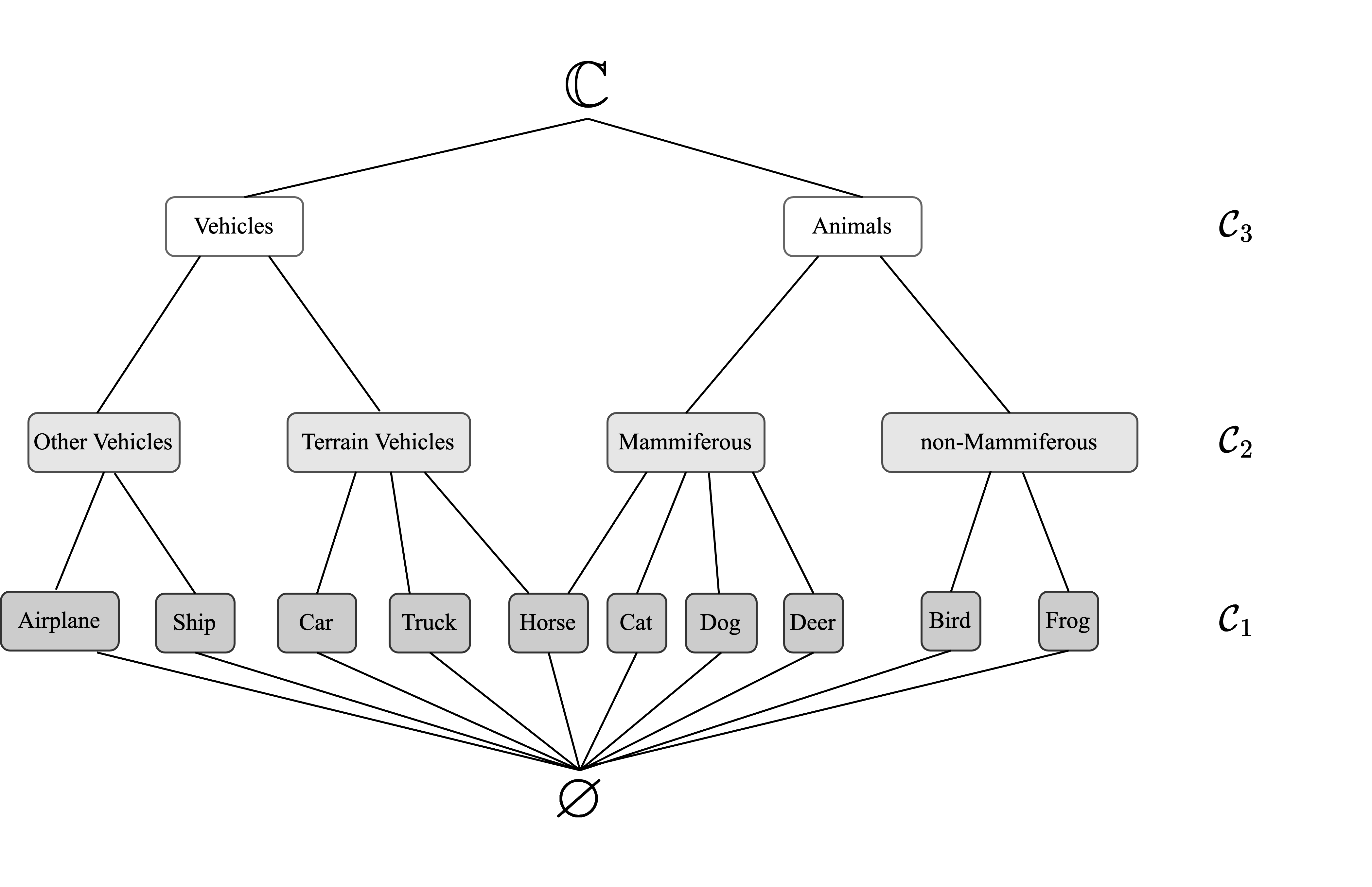}
     \vspace{-3mm}
    \caption{Abstraction hierarchy used in the CIFAR10 experiment.}
    \label{fig:hierarchy_CIFAR}
\end{figure}

Note that $\mathcal{C}_2$ ignores all the precise classes in $\mathcal{C}_1$ except \textit{Horse} being in both \textit{Terrain Vehicles} and \textit{Mammiferous}. This abstraction allows misclassifications within the same category but not between different categories. For example, misclassifying a \textit{cat} as a \textit{dog} is acceptable since both belong to the \textit{Mammiferous} category, but misclassifying a \textit{cat} as a \textit{car}, which belongs to the \textit{Terrain Vehicles} category, is not. Finally, we have the last level of aggregation, where we distinguish between \textit{animals} to \textit{vehicles}. Formally, we are considering ${\mathcal{C}_3}\defi \{\res, \mbox{Vehicle}, \mbox{Animal}, \{\mbox{Horse}\},\varnothing\}$, where from Fig.~\ref{fig:hierarchy_CIFAR} we can derive the definition of the class abstract categories. For each instance, we set a timeout verification threshold of 100 seconds.

We present the empirical results of the first experiment in Tab.~\ref{tab:result_cifar}. When verifying $\mathcal{C}_1$ (i.e., with concrete classes), increasing the \(\epsilon\) perturbation results in fewer safe instances and generally more timeouts. However, by introducing output abstraction with $\mathcal{C}_2$ or $\mathcal{C}_3$, we observe a reduction in the number of timeouts and a better reflection of the DNN's robustness. For example, with \(\epsilon = 1/255\), while standard formal verification identifies four unsafe misclassifications, the DNN reliably distinguishes between \textit{animals} and \textit{vehicles} and nearly always differentiates between the finer subcategories of $\mathcal{C}_1$, highlighting its actual robustness across these classifications. These results also support our hypothesis regarding the practical scalability of abstract verification, where providing multiple output levels can reduce the number of \textit{unknown} outcomes, thereby lowering the likelihood of incurring timeouts. This is highlighted, for instance, in Tab. \ref{tab:result_cifar}, where under an increasing perturbation of $\epsilon = 2/255$, we observe fewer timeout instances as the level of abstraction increases.
\begin{table}[t]
    \scriptsize
    \centering
    \caption{Empirical results on CIFAR10 benchmark.}
    \label{tab:result_cifar}
   \begin{tabular}{c c c c c c} 
     \Xhline{3\arrayrulewidth}
      \textbf{$\epsilon$ Perturb.} & \textbf{Abstraction} & \textbf{Safe} & \textbf{Abstract Safe} & \textbf{Unsafe} & \textbf{Timeout} \\ [0.5ex] 
     \Xhline{3\arrayrulewidth}
     1/255 & $\mathcal{C}_1$ & 46 & - & 4 & 0 \\ 
     
     1/255 & $\mathcal{C}_2$ & 46 & 2 & 2 & 0  \\
    
      1/255 & $\mathcal{C}_3$ & 46 & 4 & 0 & 0\\
     \Xhline{3\arrayrulewidth}
     2/255 & $\mathcal{C}_1$ & 27 & - & 11 & 12 \\
     
     2/255 & $\mathcal{C}_2$ & 27 & 9 & 6 & 8 \\
      
     2/255 & $\mathcal{C}_3$ & 27 & 18 & 1 & 4 \\
     \Xhline{3\arrayrulewidth}
    \end{tabular}

\end{table}
\section{Discussion}

We introduced the \textsc{Abstract DNN-Verification}, extending the standard formal verification of neural networks to include a hierarchical structure of safety and robustness properties. By allowing multiple levels of output abstraction, our approach addresses limitations in traditional verification methods, which rely on binary classifications of safe or unsafe outputs. This enhanced framework enables a more expressive analysis of deep neural networks, especially in complex scenarios where traditional safety properties are hard to write or to verify. We also establish the relationship and advancements with the concept of \textit{weakened} robustness introduced in \cite{giacobazzi2024adversities}. Importantly, we show how this formulation can be adapted to realistic benchmarks and complex tasks to rank adversarial attacks based on the impact across different output abstraction levels and get a deeper insight into the safety degree of neural networks. Future work will explore how incorporating contextual information, such as task-specific constraints, environmental cues, or user-defined priorities, can enrich the hierarchical structure of output abstractions. This could lead to adaptive verification strategies that dynamically adjust abstraction levels based on operational context, further improving interpretability and applicability in real-world scenarios. Additionally, investigating automated techniques to learn abstraction hierarchies from data could reduce manual effort and open new avenues for scalable safety analysis in deep learning systems. 

\section*{Acknowledgments}
This work has been supported by PNRR MUR project PE0000013-FAIR. The authors thank Gabriele Roncolato for his support during the empirical evaluation, running experiments, and analyzing the results.

\bibliographystyle{unsrt}  
\bibliography{references}  

\clearpage
\section*{Appendix}

\section{Computing $g^\sharp$} \label{apx:example_g_sharp}
In this section, we show a procedure for collecting the set of potential answers of an abstracted DNN semantics. Formally, let us consider a DNN with \textit{abstract semantics} $f^{\sharp}: \wp(\mathbb{R})^m\to \wp(\mathbb{R})^n$.
Let $g^{\sharp}: \wp(\mathbb{R})^n \to \wp(\res)$ with $\res \defi [1,n] \times \wp(\mathbb{R})$.
As discussed in Sect.~\ref{sec:preliminaries}, 
it is often difficult for $g^\#$ to return a set with a single index-interval tuple due to the overlap of reachable sets after the first propagation of $\mathcal{X}$ through $f^\#$, and further investigations in terms of either input space or unstable node splits are necessary. 

Nonetheless, our idea is to exploit a procedure that computes a set of \textit{potential maximum results} (instead of single ones like in the concrete semantic \cite{LiuSurvey}) and then check if the set respects at least one level of the pre-defined safety hierarchy. The approach outlined in Alg.~\ref{alg1} addresses this goal by avoiding the computation of a single reachable set and, instead, identifying the maximal set of overlapping output intervals where the output value can be selected (i.e., intervals that may contain the greatest value). This is achieved by ignoring any interval where at least one other interval has a lower bound greater than the upper bound of the considered interval. 

To provide a clarification example to the reader, suppose to have a case where the propagation of a set of input intervals $\mathcal{X}$ through the DNN $f^{\sharp}$, using interval bound propagation (IBP) \cite{lomuscio2017approach}, produces $\mathcal{R} = \tuple{[-3.6,3.2], [3.5,15.5],[3.4,6.7],[-11,-0.5], [3.2,18.2]}$. Then, if we compute the partition induced simply by the overlapping, we would have the whole set $\{c_1,c_2,c_3,c_4,c_5\}$ since there are no pairs of results with empty intersection, but $\max{(Y_4)}\leq\min{(Y_2)}$\footnote{In this context since we are considering intervals, max and min are used to collect lower and upper bounds of the intervals.}, hence $y_4$ can never be chosen, analogously $\max{(Y_1)}\leq\min{(Y_2)}$, hence also $y_2$ cannot be chosen. This means the results that can win are necessarily in the set $\{c_2,c_3,c_5\}$.


\begin{algorithm}[h!]
\caption{Computing $g^\sharp$}\label{alg1}
\begin{algorithmic}[1]
\small
\STATE \textbf{Input:} $\mathcal{R}$.
\STATE \textbf{Output: } $\wp(\res)$.
\vspace{0.2cm}
\STATE $R = \wp(\res)$
\FOR{$i\in[1,n]$}
    \FOR{$j\in[1,n],j\neq i$}
        \IF{$\max{(Y_i)}\leq\min{(Y_j)}$}
            \STATE $R=R\smallsetminus \langle i,Y_i \rangle$
            \STATE \textbf{exitfor}
        \ENDIF
    \ENDFOR
    
\ENDFOR
\STATE \textbf{return} $R$
\end{algorithmic}
\end{algorithm}

\section{\textit{Coherence} Does Not Imply \textit{Abstract Coherence}: an Example}\label{apx:example_proof}

Our proof of Theorem \ref{thm:abs_coherence} demonstrates that verifying \textit{Coherence} on specific input perturbations and concrete semantics does not guarantee the DNN's overall abstract safety level. This issue stems from the pointwise verification \cite{LiuSurvey} nature of \textit{Coherence}, which may overlook critical cases where the DNN violates abstract output properties. In contrast, \textit{Abstract Coherence}, by exploiting an abstracted input perturbation and execution framework, provides a more robust and reliable safety guarantee. To see this in practice, let us consider the following example. 

\begin{figure}[h!]
	\begin{center}
		\scalebox{0.8} {
			\def\layersep{3cm}
			\begin{tikzpicture}[shorten >=1pt,->,draw=black!50, node
				distance=\layersep,font=\footnotesize]
				
				\node[input neuron] (I-1) at (0,-1) {$x_1$};
				\node[input neuron] (I-2) at (0,-2.5) {$x_2$};
        
				\node[left=-0.05cm of I-1] (b1) {$X_1 = (0, 1]$};
				\node[left=-0.05cm of I-2] (b2) {$X_2=(0, 1]$};

				\node[hidden neuron] (H-1) at (1.2*\layersep,-1) {};
				\node[hidden neuron] (H-2) at (1.2*\layersep,-2.5) {};

				\node[hidden neuron] (H-3) at (1.8*\layersep,-1) {};
				\node[hidden neuron] (H-4) at (1.8*\layersep,-2.5) {};

			    \node[output neuron] at (3*\layersep, 0) (O-0) {};
                \node[output neuron] at (3*\layersep, -1.5) (O-1) {};
                \node[output neuron] at (3*\layersep, -3) (O-2) {};

                \node[below=0.05cm of O-0]{$y_1$};
                \node[below=0.05cm of O-1]{$y_2$};
                \node[below=0.05cm of O-2]{$y_3$};

                \node[right=0.05cm of O-0] (b0) {};
                \node[right=0.05cm of O-1] (b1) {};
                \node[right=0.05cm of O-2] (b2) {};
               
				
				\draw[nnedge] (I-1) --node[above, pos=0.25] {$1$} (H-1);
				\draw[nnedge] (I-1) --node[above, pos=0.25] {$2$} (H-2);
                
				\draw[nnedge] (I-2) --node[above, pos=0.1] {$2$} (H-1);
                \draw[nnedge] (I-2) --node[above, pos=0.25] {$-3$} (H-2);

				\draw[nnedge] (H-1) --node[above] {ReLU} (H-3);
				\draw[nnedge] (H-2) --node[below] {ReLU} (H-4);

                \draw[nnedge] (H-3) --node[above, pos=0.7] {$0.8$} (O-0);
				\draw[nnedge] (H-3) --node[above, pos=0.75] {$-0.1$} (O-1);
                \draw[nnedge] (H-3) --node[above, pos=0.8] {$0.5$} (O-2);

                \draw[nnedge] (H-4) --node[above, pos=0.65] {$0.3$} (O-0);
				\draw[nnedge] (H-4) --node[above, pos=0.7] {$-1$} (O-1);
                \draw[nnedge] (H-4) --node[above, pos=0.7] {$0.6$} (O-2);


				\node[below=0.05cm of H-1] (b1) {};
				\node[below=0.05cm of H-2] (b2) {};

				\node[below=0.05cm of H-3] (b1) {};
				\node[below=0.05cm of H-4] (b2) {};


		
				\node[annot,above of=H-1, node distance=0.8cm] (hl1) {Weighted
					sum};
				\node[annot,above of=H-3, node distance=0.8cm] (hl2) {Activated Layer };
				\node[annot,above of=I-1, node distance=0.8cm] {Input };
				\node[annot,above of=O-0, node distance=0.8cm] {Output };
			\end{tikzpicture}
		} 
		
	\end{center}
    \caption{Toy DNN for this example. $\res\defi\{c_i\defi\tuple{i,Y_i}~|~i\in[1,3]\}$.}
    \label{fig:example_proof}
\end{figure}

Consider the DNN depicted in Fig. \ref{fig:example_proof}
where we have $x_1,x_2 \in [0,1]$, namely $\mathcal{X} \defi\langle [0,1], [0,1] \rangle$, and $\mathcal{S} = \{c_1\} \subsetneq \res$. We then define $\mathcal{C}=\{\res, \{c_1, c_2\}, \mathcal{S}, \varnothing\}$, meaning that we tolerate a classification that consider both $c_1$ and $c_2$ but $c_3$ is not accepted as considered unsafe. We start by verifying if \textit{Coherence} is respected. 

Consider an input $\mathbf{x} = \langle 0.5, 0.5 \rangle \in \mathcal{X}$. Let us define $\Im_{\varepsilon}(\mathbf{x})$ as  $\ell_\infty$-ball perturbation of radius $\epsilon$ around $\mathbf{x}$; formally 
$\Im_{\varepsilon}(\mathbf{x}) = \{ \mathbf{x}' \in \mathbb{R}^m \, | \, \|\mathbf{x}' - \mathbf{x}\|_\infty \leq \varepsilon \}$.


For this example, let us consider an $\varepsilon$ radius for the perturbation equals $0.5$ around $\mathbf{x}$ for the \textit{Coherence} test. We start by noticing the first limitation of \textit{Coherence} related to the necessity of executing the test on the DNN's concrete semantics. Specifically, even if we consider, for instance, a subset of the perturbation $\Im_{\mbox{\tiny $0.5$}}(\mathbf{x})$ that considers only the inputs of the same values, namely $\Im_{\mbox{\tiny $0.5$}}(\mathbf{x}) = \sset{\tuple{x,x}}{x\in(0,1]} \subseteq \mathbb{R}^2$, we have a infinity number of couples to test, which prevents the possibility to state whether \textit{Coherence} holds or not. To address this issue, one possible solution is to assume some discretization of the input space (e.g., to the maximum resolution allowed by the machine precision), such that $\Im_{\mbox{\tiny $0.5$}}(\mathbf{x})$ becomes a finite set, that can be easily tested even in the concrete semantics. 

Nonetheless, in this specific example for the \textit{Coherence} test, thanks to the simple nature of the DNN, we can employ the symbolic propagation \cite{reluval} to compute linear equations of each layer w.r.t the previous one. In detail can be easily verified by direct inspection of the possible output values in the DNN of Fig. \ref{fig:SIP_proof} that even if $\Im_{\mbox{\tiny $0.5$}}(\mathbf{x}) = \sset{\tuple{x,x}}{x\in(0,1]} \subseteq \mathbb{R}^2$ has infinite cardinality, for all $x_1 = x_2 \in (0,1]$ the final result chosen by the DNN will always be $y_1$, as it presents the maximum value w.r.t the other ones.

\begin{figure}[h!]
	\begin{center}
		\scalebox{0.8} {
			\def\layersep{3cm}
			\begin{tikzpicture}[shorten >=1pt,->,draw=black!50, node
				distance=\layersep,font=\footnotesize]
				
				\node[input neuron] (I-1) at (0,-1) {$x_1$};
				\node[input neuron] (I-2) at (0,-2.5) {$x_2$};
        
				\node[left=-0.05cm of I-1] (b1) {$X_1 = (0, 1]$};
				\node[left=-0.05cm of I-2] (b2) {$X_2=(0, 1]$};

				\node[hidden neuron] (H-1) at (1.2*\layersep,-1) {};
				\node[hidden neuron] (H-2) at (1.2*\layersep,-2.5) {};

				\node[hidden neuron] (H-3) at (1.8*\layersep,-1) {};
				\node[hidden neuron] (H-4) at (1.8*\layersep,-2.5) {};

			    \node[output neuron] at (3*\layersep, 0) (O-0) {};
                \node[output neuron] at (3*\layersep, -1.5) (O-1) {};
                \node[output neuron] at (3*\layersep, -3) (O-2) {};

                \node[below=0.05cm of O-0]{$y_1$};
                \node[below=0.05cm of O-1]{$y_2$};
                \node[below=0.05cm of O-2]{$y_3$};

                \node[right=0.05cm of O-0] (b0) {$0.8x_1+1.6x_2$};
                \node[right=0.05cm of O-1] (b1) {$-0.1x_1-0.2x_2$};
                \node[right=0.05cm of O-2] (b2) {$0.5x_1+x_2$};
               
				
				\draw[nnedge] (I-1) --node[above, pos=0.25] {$1$} (H-1);
				\draw[nnedge] (I-1) --node[above, pos=0.25] {$2$} (H-2);
                
				\draw[nnedge] (I-2) --node[above, pos=0.1] {$2$} (H-1);
                \draw[nnedge] (I-2) --node[above, pos=0.25] {$-3$} (H-2);

				\draw[nnedge] (H-1) --node[above] {ReLU} (H-3);
				\draw[nnedge] (H-2) --node[below] {ReLU} (H-4);

                \draw[nnedge] (H-3) --node[above, pos=0.7] {$0.8$} (O-0);
				\draw[nnedge] (H-3) --node[above, pos=0.75] {$-0.1$} (O-1);
                \draw[nnedge] (H-3) --node[above, pos=0.8] {$0.5$} (O-2);

                \draw[nnedge] (H-4) --node[above, pos=0.65] {$0.3$} (O-0);
				\draw[nnedge] (H-4) --node[above, pos=0.7] {$-1$} (O-1);
                \draw[nnedge] (H-4) --node[above, pos=0.7] {$0.6$} (O-2);


				\node[below=0.05cm of H-1] (b1) {$x_1+2x_2$};
				\node[below=0.05cm of H-2] (b2) {$2x_1-3x_2$};

				\node[below=0.05cm of H-3] (b1) {$x_1+2x_2$};
				\node[below=0.05cm of H-4] (b2) {$0$};


		
				\node[annot,above of=H-1, node distance=0.8cm] (hl1) {Weighted
					sum};
				\node[annot,above of=H-3, node distance=0.8cm] (hl2) {Activated Layer };
				\node[annot,above of=I-1, node distance=0.8cm] {Input };
				\node[annot,above of=O-0, node distance=0.8cm] {Output };
			\end{tikzpicture}
		} 
		
	\end{center}
    \caption{Symbolic propagation of $\tuple{x_1,x_2}$ through the DNN considered in this example.}
    \label{fig:SIP_proof}
\end{figure}
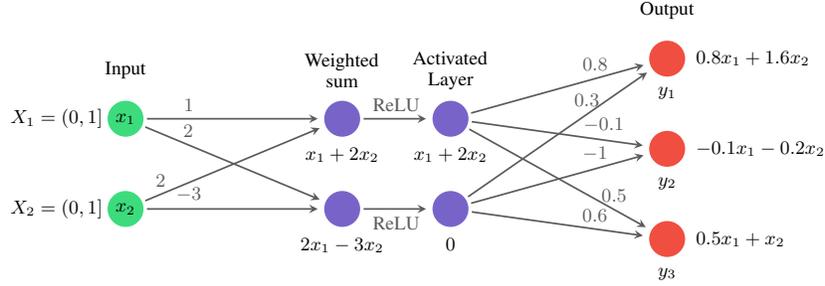

Hence, from the fact that $\mathcal{C}\comp(g \circ f)\comp\Im_{\mbox{\tiny $0.5$}}(\mathbf{x}) = \{c_1\} = \mathcal{S} \subsetneq \res$, we conclude \textit{Coherence} is satisfied. 

If we now consider the widening input perturbation function $\wIm_{\mbox{\tiny $0.5$}}(\mathbf{x}) \defi \mathcal{X} = \tuple{[0,1],[0,1]}$ (considering, for a fair comparison, the same range of values as $\Im_{\mbox{\tiny $0.5$}}$, but now all the possible permutations in $[0,1]\times[0,1]$ and not only same couple of values $x_1=x_2 \in (0,1]$), we observe that \textit{Abstract Coherence} is not guaranteed. Specifically, our novel formulation of perturbation function $\wIm$ and abstract DNN semantics allows us to propagate the entire abstract input domain $\mathcal{X}$ through $f^\sharp$ (using, for instance, IBP as in the DNN of Fig. \ref{fig:IBP_example_proof}), obtaining the reachable set $\mathcal{R} = \langle [0, 3.0], [-2.3, 0], [0, 2.7] \rangle$. Applying $g^\sharp$ to $\mathcal{R}$ yields the set $\{c_1, c_3\}$. Consequently, the result of $\mathcal{C}(g^\sharp(\mathcal{R})) = \res$, violates the \textit{Abstract Coherence} requirement.

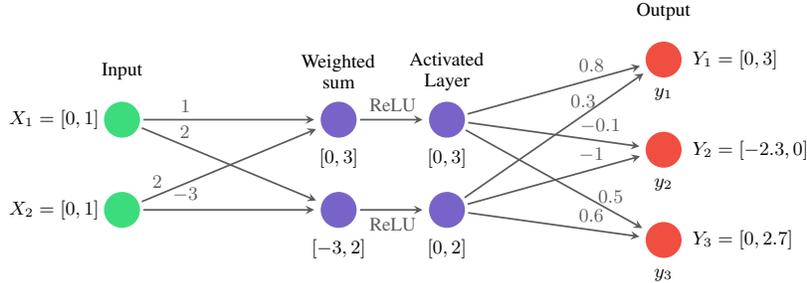
\begin{figure}[h!]
	\begin{center}
		\scalebox{0.8} {
			\def\layersep{3cm}
			\begin{tikzpicture}[shorten >=1pt,->,draw=black!50, node
				distance=\layersep,font=\footnotesize]
				
				\node[input neuron] (I-1) at (0,-1) {};
				\node[input neuron] (I-2) at (0,-2.5) {};
        
				\node[left=-0.05cm of I-1] (b1) {$X_1 = [0, 1]$};
				\node[left=-0.05cm of I-2] (b2) {$X_2=[0, 1]$};

				\node[hidden neuron] (H-1) at (1.2*\layersep,-1) {};
				\node[hidden neuron] (H-2) at (1.2*\layersep,-2.5) {};

				\node[hidden neuron] (H-3) at (1.8*\layersep,-1) {};
				\node[hidden neuron] (H-4) at (1.8*\layersep,-2.5) {};

			    \node[output neuron] at (3*\layersep, 0) (O-0) {};
                \node[output neuron] at (3*\layersep, -1.5) (O-1) {};
                \node[output neuron] at (3*\layersep, -3) (O-2) {};

                \node[below=0.05cm of O-0]{$y_1$};
                \node[below=0.05cm of O-1]{$y_2$};
                \node[below=0.05cm of O-2]{$y_3$};

                \node[right=0.05cm of O-0] (b0) {$Y_1=[0, 3]$};
                \node[right=0.05cm of O-1] (b1) {$Y_2=[-2.3, 0]$};
                \node[right=0.05cm of O-2] (b2) {$Y_3=[0, 2.7]$};
               
				
				\draw[nnedge] (I-1) --node[above, pos=0.25] {$1$} (H-1);
				\draw[nnedge] (I-1) --node[above, pos=0.25] {$2$} (H-2);
                
				\draw[nnedge] (I-2) --node[above, pos=0.1] {$2$} (H-1);
                \draw[nnedge] (I-2) --node[above, pos=0.25] {$-3$} (H-2);

				\draw[nnedge] (H-1) --node[above] {ReLU} (H-3);
				\draw[nnedge] (H-2) --node[below] {ReLU} (H-4);

                \draw[nnedge] (H-3) --node[above, pos=0.7] {$0.8$} (O-0);
				\draw[nnedge] (H-3) --node[above, pos=0.75] {$-0.1$} (O-1);
                \draw[nnedge] (H-3) --node[above, pos=0.8] {$0.5$} (O-2);

                \draw[nnedge] (H-4) --node[above, pos=0.65] {$0.3$} (O-0);
				\draw[nnedge] (H-4) --node[above, pos=0.7] {$-1$} (O-1);
                \draw[nnedge] (H-4) --node[above, pos=0.7] {$0.6$} (O-2);


				\node[below=0.05cm of H-1] (b1) {$[0, 3]$};
				\node[below=0.05cm of H-2] (b2) {$[-3, 2]$};

				\node[below=0.05cm of H-3] (b1) {$[0, 3]$};
				\node[below=0.05cm of H-4] (b2) {$[0, 2]$};


		
				\node[annot,above of=H-1, node distance=0.8cm] (hl1) {Weighted
					sum};
				\node[annot,above of=H-3, node distance=0.8cm] (hl2) {Activated Layer };
				\node[annot,above of=I-1, node distance=0.8cm] {Input };
				\node[annot,above of=O-0, node distance=0.8cm] {Output };
			\end{tikzpicture}
		} 
		
	\end{center}
    \caption{Interval bound propagation of $\mathcal{X}=\tuple{[0,1], [0,1]}$ through the DNN considered in this example.}
    \label{fig:IBP_example_proof}
\end{figure}

To confirm this result, we employed the verification tool proposed in \cite{eProve} to analyze whether, for any $\mathbf{x} \in \mathcal{X}$, the DNN avoids selecting the unsafe output $y_3$, which conflicts with the abstraction $\mathcal{C}$. The tool identified a counterexample: $\mathbf{x} = \tuple{0.2808, 0.0508}$, which produces the output $\mathbf{y} = \tuple{0.4287, -0.4475, 0.4368}$. Evaluating $g(\mathbf{y})$ results in $\{c_3\}$ (unsafe), thereby confirming that \textit{Abstract Coherence}, and safety in general, are not guaranteed for any $\mathbf{x} \in \mathcal{X}$.

This simple example highlights a key limitation of \textit{Coherence}: the $\Im$ function employed in the formulation considers only a perturbation around a concrete fixed point $\mathbf{x}$. This restricted focus can mask potential unsafe behaviors that might emerge when the entire abstract domain $\mathcal{X}$ is propagated through the network. In contrast, our formulation starts from the entire abstract domain $\mathcal{X}$ and allows considering a more general abstract input by leveraging the widening input perturbation function $\wIm$ and abstract reasoning about the DNN's semantics. This demonstrates how \textit{Abstract Coherence} offers guarantees over the entire abstract input domain rather than specific regions, thereby addressing the limitations of \textit{Coherence}.

\section{On the Hardness of Abstract DNN-Verification Problem}\label{apx:hardness}
Although lifting the output of $g^\#$ can yield more informative answers with potentially fewer verification queries, the \textsc{Abstract DNN-Verification} problem remains difficult to solve. Even with a large abstraction of tolerable answers, the number of iterative refinements required to reach a solution can still grow exponentially with the size of the instance. As a result, we have the following.

\setcounter{proposition}{2}
\begin{proposition}
    The \textsc{Abstract DNN-Verification} problem is NP-Complete.
\end{proposition}

The inclusion in NP is straightforward to prove. A certificate for this problem is a specific set of abstracted reachable outputs resulting from the propagation of the perturbed inputs $\wIm(\mathcal{X})$ through $f^\sharp$ and $g^\sharp$. This set of outputs can be checked to see if $\mathcal{C}(g^\#(\mathcal{R})) \subsetneq \res$. Given the approximated neural network function $f^\sharp$, we can compute the abstract reachable set of outputs $f^\sharp(\wIm(\mathcal{X}))$ using interval analysis or zonotope propagation, which are polynomial in the size of the network. Once the abstract output set from $f^\sharp(\wIm(\mathcal{X}))$ is obtained, we apply the procedure of Alg.~\ref{alg1} to compute $g^\sharp$, polynomial in the number of reachable sets. Finally, we check if the abstract reachable set of outputs $g^\sharp \circ f^\sharp(\wIm(\mathcal{X}))$ is contained in the abstract safe set $\mathcal{C}$. This inclusion check is usually a simple geometric containment problem (e.g., verifying if one interval or polytope is contained in another), which can indeed be done in polynomial time. Hence, the problem is in NP.

The hardness of the problem directly follows from the result of \cite{Reluplex}, showing that there exists a reduction from \textsc{3-SAT} to the \textsc{DNN-Verification} problem.

\section{Empirical Evaluation: Further Details}\label{appendix:full_empirical}

\subsection*{Habitat-Lab Experiment}
Habitat \cite{habitat2,habitat3} is a high-performance simulator designed for training and evaluating deep reinforcement learning (DRL) models, particularly in 3D environments for embodied AI tasks. In the \textit{Social Navigation} task employed in our evaluation, we have a robotic agent that has to follow a humanoid without any collision. This evaluation considers the publicly released trained agent from the Habitat-Lab GitHub repository.\footnote{\url{https://github.com/facebookresearch/habitat-lab}} The agent's input space consists of an arm depth camera, a binary human detector, and depth stereo cameras, while its output space controls the agent's linear and angular velocities.

\paragraph{Hierarchy Abstraction.} This experiment shows how the \textsc{ADV} problem can be applied to continuous output spaces by selecting a subset of discrete velocity classes. Specifically, the original output space for the robotic agent includes linear and angular velocities in $\mathbb{R}$. To create a hierarchical abstraction of the output space, we first discretize it at the desired maximum resolution, then divide the actions into a set of discrete velocity classes, denoted as $\textbf{V}$. Each class $V_i$ in $\textbf{V}$ is defined as $V_i = \{(v, w) \mid v_{i_{\text{min}}} \leq v \leq v_{i_{\text{max}}}, w_{i_{\text{min}}} \leq w \leq w_{i_{\text{max}}}\}$, where $v$ and $\omega$ represent the linear and angular velocities, respectively, bounded by predefined lower and upper thresholds. Using these discrete classes, we can construct the first level of hierarchical abstraction, $\mathcal{C}_1$, as illustrated in Fig. \ref{fig:hierarchy_habitat}, in the main paper. We can further aggregate elements from $\mathcal{C}_1$, resulting in an abstract classification within $\mathcal{C}_2$.

Fig. \ref{fig:hierarchy_habitat}, in the main paper, reports the complete output abstraction employed for \textit{Situation 1}. For the sake of simplicity, we only report the abstraction for the first unsafe situation, but a similar strategy is applied to all the other situations considered using different thresholds as described in Tab.~\ref{tab:concrete_classes_habitat}.

\begin{table}[h!]
    \small
    \centering
    \caption{Definition of the safe and abstract safe output bounds for the Habitat experiments. Here, $-\infty, +\infty$ indicates any possible negative and positive robot velocities, respectively.}
    \label{tab:concrete_classes_habitat}
   \begin{tabular}{c c c} 
     \Xhline{3\arrayrulewidth}
     \textbf{Situation} & \textbf{Safe output} ($v$, $\omega$) & \textbf{Abstract safe output} ($v$, $\omega$) \\ [0.5ex] 
     \Xhline{3\arrayrulewidth}
      1 & ($-\infty$, $-0.4$], [$0.0$, $0.1$] & ($-\infty$, $-0.4$], ($-\infty$, $+\infty$) \\
      2 & [$0.4$, $+\infty$), [$0.0$, $0.1$] & [$0.0$, $+\infty$), ($-\infty$, $+\infty$) \\ 
      3 & [$0.0$, $+\infty$), [$0.1$, $+\infty$) & [$-0.4$, $+\infty$), ($-\infty$, $+\infty$) \\
    \Xhline{3\arrayrulewidth}
    \end{tabular}
\end{table}

\end{document}